\documentclass[twoside]{article}

\usepackage[accepted]{aistats2018}
\usepackage{xcolor}
\usepackage{colortbl}
\usepackage{graphicx}

\makeatletter
\newcommand{\thickhline}{%
    \noalign {\ifnum 0=`}\fi \hrule height 1pt
    \futurelet \reserved@a \@xhline
}

\newcolumntype{"}{@{\hskip\tabcolsep\vrule width 1pt\hskip\tabcolsep}}
\makeatother

\def\grayOne{\color{black!10}1\color{black}}

\def\0{\color{white}1\color{black}}

\def\grayBlock{\cellcolor{black!10}\grayOne}

\usepackage{amsmath}
\usepackage{enumitem}
\usepackage{amsthm}
\usepackage{amssymb}
\usepackage{color}		   
\usepackage{graphicx}
\usepackage[plain,linesnumbered,ruled]{algorithm2e}
\usepackage{array}
\usepackage{multirow}
\usepackage{soul}
\usepackage[normalem]{ulem}
\usepackage{ctable} 
\usepackage{bm}
\usepackage{caption}
\usepackage{subcaption}
\usepackage{floatrow}

\newtheorem{thm}{Theorem}
\newtheorem{observation}{Observation}
\newtheorem{mydef}{Definition}
\newtheorem{corollary}{Corollary}
\newtheorem{lemma}{Lemma}

\def\pp{p_{\mathrm{in}}}
\def\qp{p_{\mathrm{out}}}

\def\one{\mathbf{1}}
\def\diag{\mathrm{diag}}

\def\x{\mathbf x} \def\y{\mathbf y}  
\def\u{\mathbf u} \def\v{\mathbf v}  
 \def\e{\mathbf e}

\def\sym{\mathrm{sym}}

\def\PowerMeanLaplacian{L_{p}}

\def\R{\mathbb{R}}

\newcommand{\pin}{p_{in}}
\newcommand{\pout}{p_{out}}

\newcommand{\ones}{\one}

\newcommand{\sign}{\operatorname{sign}}

\newcommand{\Wt}{\mathcal{W}^{(t)}}
\newcommand{\bchi}{\boldsymbol{\chi}}
\newcommand{\lsymn}[1]{\tilde{\mathcal{L}}_{\operatorname{sym}}^{(#1)}}

\newcommand{\Lp}[1]{\mathsf{L}_{#1}}
\newcommand{\Lnp}[1]{\mathsf{\tilde L}_{#1}}

\newcommand{\PowerMeanLaplacianCommand}[1]{L_{#1}}
\newcommand{\PowerMeanLaplacianInExpectationCommand}[1]{\mathcal{L}_{#1}}

\newcommand{\abs}[1]{\left|#1\right|}
\newcommand{\vectornorm}[1]{\left|\left|#1\right|\right|}
\newcommand{\multiLayerGraph}[1]{\mathbb{#1}}

 \newenvironment{talign*}
 {\csname align*\endcsname}
 {\endalign}

   \newcommand{\myComment}[1]{}  

\def\pedro#1{\textcolor{black}{#1}}

\renewcommand{\L}[1]{\mathcal{L}_{#1}}
\newcommand{\Ln}[1]{\mathcal{\tilde L}_{#1}}
\newcommand{\lsym}[1]{L_{sym}^{(#1)}}

\usepackage{subcaption}
\usepackage{enumitem}
\begin{document}

%

%
\runningauthor{Pedro Mercado, Antoine Gautier, Francesco Tudisco, Matthias Hein}

\twocolumn[

\aistatstitle{The Power Mean Laplacian for Multilayer Graph Clustering}

\aistatsauthor{ Pedro Mercado$^1$ \And Antoine Gautier$^1$ \And  Francesco Tudisco$^2$ \And Matthias Hein$^1$ }

\aistatsaddress{ 
$^1$Department of Mathematics and Computer Science, Saarland University, Germany \\
$^2$Department of Mathematics and Statistics, University of Strathclyde, G11XH Glasgow, UK
} 
]

\begin{abstract}

Multilayer graphs encode different kind of interactions between the same set of entities.
When one wants to cluster such a multilayer graph, the natural question arises how one should merge
the information \pedro{from} different layers. We introduce in this paper a one-parameter family of matrix power means
for merging the Laplacians from different layers and analyze it in expectation in the stochastic block model. We show that
this family allows to recover ground truth clusters under different settings and verify this in real world data.
%
%
While computing the matrix power mean can be very expensive for large graphs, we introduce a numerical scheme to efficiently compute its eigenvectors for the case of large sparse graphs.
\end{abstract}

\section{Introduction}\label{sec:introduction}

Multilayer graphs have received an increasing amount of attention due to their capability to encode different kinds of interactions between
the same set of entities~\cite{Boccaletti:2014a,Kivela:2016:Multilayer}.  
This kind of graphs arise naturally in diverse applications such as
transportation networks~\cite{gallotti2015multilayer},
financial-asset markets~\cite{Bazzi:2016:Community},
temporal dynamics~\cite{taylor:2016b, Taylor:2016a},
semantic world clustering~\cite{sedoc:2017a},
multi-video face analysis~\cite{Cao:2015a},
mobile phone networks~\cite{kiukkonen:2010a},
social balance~\cite{Cartwright:1956:Structural},
citation analysis~\cite{Tang:2009:CMG:1674659.1677085}, and many others.
The extension of clustering techniques to multilayer graphs 
is a challenging task and several approaches have been proposed so far.
See 
\cite{Kim:2015:CDM,Sun:2013a,Xu:2013a,zhao:2017b}
for an overview.
For instance, 
\cite{Dong:2012:Clustering,Dong:2014:Grassmann,Tang:2009:CMG:1674659.1677085,zhao:2017a}
rely on matrix factorizations, whereas
\cite{Bacco:2017,Paul:2016b,Peixoto:2015a,Schein:2015:BPT:2783258.2783414,schein:2016}
take a Bayesian inference approach,
and~\cite{Kumar:2011,NIPS2011_4360} enforce consistency among layers in the resulting clustering assignment.
In 
\cite{Mucha:2010a,Paul:2016a,wilson2017community}
Newman's modularity~\cite{Newman:2006a} is extended to multilayer graphs.
Recently 
\cite{Domenico:2015a,Stanley:2016a}
proposed to compress a multilayer graph by combining sets of similar layers (called \textquoteleft strata\textquoteright) to later identify the corresponding communities. 
Of particular interest to our work is the popular 
approach~\cite{Argyriou:2005,Chen:2017a,huang:2012a,taylor:2016b,zhou2007spectral} that first blends the information of a multilayer graph by
finding a suitable weighted arithmetic mean of the layers and then apply standard clustering methods to the resulting mono-layer graph.

In this paper we focus on extensions of spectral clustering to multilayer graphs.
Spectral clustering is a well established method for one-layer graphs which, based on the first eigenvectors of the graph Laplacian, embeds nodes of the graphs in $\R^k$  and then uses $k$-means to find the partition. 
We propose to blend the information of a multilayer graph by taking certain matrix power means of Laplacians of the layers.

The power mean of scalars is a general family of means that includes as special cases,
the arithmetic, geometric and harmonic means. The arithmetic mean of Laplacians has been used before \pedro{in the case of signed networks}~\cite{Kunegis:2010:spectral}
and thus our family of matrix power means, see Section \ref{subsec:matrix_power_means_for_multilayer_graph_clustering}, is a natural extension of this approach. One of our main contributions is to show that the arithmetic mean is actually suboptimal to merge information from different
layers.  

We analyze the family of matrix power means in the Stochastic Block Model (SBM) for multilayer graphs in two settings, see Section \ref{sec:stochastic_block_model}. In the first one all the layers are informative, whereas in the second setting none of the individual layers contains the full information but only if one considers them all together.  We show that as the parameter of the family of Laplacian means tends to $-\infty$,  in expectation one can recover perfectly
the clusters in both situations. We provide extensive experiments which show that this behavior is stable when
one samples sparse graphs from the SBM.
Moreover, in Section \ref{sec:experiment}, we provide additional experiments on real world graphs which confirm our finding in the SBM.

A main challenge for our approach is that the matrix power mean of sparse matrices is in general dense and thus does not scale to large sparse networks in a straightforward fashion. 
Thus a further contribution of this paper in Section \ref{sec:PowerMethod} is to show that the first few eigenvectors of the matrix power mean  can be computed efficiently.  Our algorithm combines the power method with a Krylov subspace approximation technique  and allows to  compute the extremal eigenvalues and eigenvectors of the power mean of matrices without ever computing the matrix itself.


\section{Spectral clustering of multilayer graphs using matrix power means of Laplacians}\label{sec:multilayer_graph_clustering}
Let $V=\{v_1,\ldots,v_n\}$ be a set of nodes and let $T$ the number layers,
represented by adjacency  matrices $\multiLayerGraph{W} = \{ W^{(1)},\ldots,W^{(T)} \}$.
For each non-negative weight matrix $W^{(t)} \in \R^{n\times n}_+$ we have a graph $G^{(t)} = (V,W^{(t)})$  and a multilayer graph is the set ${\multiLayerGraph{G}=\{G^{(1)}, \ldots,G^{(T)}\} }$.  
%
%
%
%
%
In this paper our main focus are assortative graphs. This kind of graphs are the most common in the literature (see f.i.~\cite{Luxburg:2007:tutorial}) and are used to model the situation where edges carry \textit{similarity} information of pairs of vertices and thus are indicative for vertices being in the same cluster. For an assortative graph $G=\!(V,W)$ spectral clustering is typically based on the Laplacian matrix and its normalized version, defined respectively as

$\,$\hfill  $L = D - W \quad\qquad\,\, L_\sym = D^{-1/2}LD^{-1/2}$ \hfill $\,$

where $D_{ii} = \sum_{j=1}^n w_{ij}$ is the diagonal matrix of the degrees of $G$.
Both Laplacians are symmetric positive semidefinite and the multiplicity of eigenvalue $0$ is
equal to the number of connected components in $G$.

Given a multilayer graph with all assortative layers $G^{(1)}, \dots, G^{(T)}$, our goal is to come up with a clustering of the vertex set $V$.  We point out that in this paper a clustering is a partition of $V$, that is each vertex is uniquely assigned to one cluster.

\subsection{Matrix power mean of Laplacians for multilayer graphs}\label{sec:power_mean_of_laplacians}

Let us briefly recall the scalar power mean of a set of non-negative scalars $x_1 , \ldots , x_T$. This is a general one-parameter family of means defined for $p\in\mathbb{R}$ as $m_{p}(x_1 , \ldots , x_T)=({\frac  {1}{T}}\sum _{{i=1}}^{T}x_{i}^{p})^{{{1/p}}}$. 
It includes some well-known means as special cases:
\begin{talign*}
 \lim\limits_{p\to\infty} m_p(x_1,\dots,x_T) &= \max \{x_1,\dots,x_T\}   \\
                   m_1(x_1,\dots,x_T) &= (x_1 + \dots + x_T)/T   \\
 \lim\limits_{p\to0} m_p(x_1,\dots,x_T)      &= \sqrt[T]{x_1\cdot\dots\cdot x_T} \\
                m_{-1}(x_1,\dots,x_T) &= T\, (\frac{1}{x_1}+\dots+\frac{1}{x_T})^{-1} \\ 
\lim\limits_{p\to-\infty} m_p(x_1,\dots,x_T) &= \min \{x_1,\dots,x_T\} 
\end{talign*}
corresponding to the maximum, arithmetic, geometric, harmonic mean and minimum, respectively.


Since matrices do not commute, the scalar power mean can be extended to positive definite matrices in a number of different ways, all of them coinciding when applied to  commuting matrices. In this work we use the following matrix power mean.  
\begin{mydef}[\cite{bhagwat_subramanian_1978}]\label{definition:MatrixPowerMean}
Let $A_1 , \ldots , A_T$ be symmetric positive definite matrices, and $p\in\mathbb{R}$.
The matrix power mean of $A_1 , \ldots , A_T$ with exponent $p$ is
\begin{equation}\label{eq:MatrixPowerMean}
 \textstyle{M_{p}(A_1 , \ldots , A_T)=\left({\frac  {1}{T}}\sum _{{i=1}}^{T}A_{i}^{p}\right)^{ 1/p  } }
\end{equation}
where $A^{1/p}$ is the unique positive definite solution of the matrix Equation $X^p = A$.
\end{mydef}
\pedro{
The previous definition can be extended to positive semi-definite matrices.
For $p\!>\!0$, $M_{p}(A_1 , \ldots , A_T)$ exists for positive semi-definite matrices,
whereas for $p\!\leq\!0$ it is necessary to add a suitable diagonal shift to $A_1 , \ldots , A_T$ to enforce them to be positive definite (see~\cite{bhagwat_subramanian_1978} for details).
}

We call the matrix above \textit{matrix power mean} and we recover for $p=1$ the standard arithmetic mean of the matrices. Note that for $p\to 0$, the power mean \eqref{eq:MatrixPowerMean} converges to the Log-Euclidean matrix  mean \cite{arsigny2007geometric}

$\,$\hfill  $M_{0}(A_1 , \ldots , A_T) = \exp \left({\frac  {1}{T}}\sum _{{i=1}}^{T}\log A_{i}\right)\, ,$ \hfill $\,$

{which is a popular form of matrix geometric mean used, for instance, in diffusion tensor imaging or quantum information theory (see f.i. \cite{arsigny2006log, petz2007quantum}).  

Based on the Karcher mean, a different one-parameter family of matrix power means has been discussed for instance in \cite{lim2012matrix}. When the parameter goes to zero, the Karcher-based power mean of two matrices $A$ and $B$ converges to the geometric mean $A\# B = A^{1/2}(A^{-1/2}BA^{-1/2})^{1/2}A^{-1/2}$. 
The  mean $A\# B$ has been used for instance in \cite{fasi:2016:computing,Mercado:2016:Geometric} for clustering in signed networks,
for metric learning~\cite{zadeh2016geometric} and geometric optimization~\cite{sra2016geometric}.
However, when more than two matrices are considered, the Karcher-based power mean is defined as the  solution of a set of nonlinear matrix equations with no known closed-form solution and thus is not suitable for multilayer graphs.

The matrix power mean \eqref{eq:MatrixPowerMean} is symmetric positive definite and is independent of the labeling of the vertices in the sense that the matrix power mean of relabeled matrices is the same as relabeling the matrix power mean of the original matrices.
The latter property is a necessary requirement for any clustering method.
The following lemma illustrates the relation to the scalar power mean and is frequently used in the proofs.
\begin{lemma}\label{lemma:eigenvalues_and_eigenvectors_of_generalized_mean_V2}
Let $\u$ be an eigenvector of $A_1 , \ldots , A_T$ with corresponding eigenvalues $\lambda_1 , \ldots , \lambda_T$.
Then $\u$ is an eigenvector of $M_p(A_1 , \ldots , A_T)$ with eigenvalue $m_p(\lambda_1 , \ldots , \lambda_T)$.
\end{lemma}
\myComment{
\begin{proof}
Observe that for any positive definite matrix $M$, if $Mx = \lambda(M)x$, then $M^p=\lambda(M)^px$.
 Thus, we can see that as
 $A_i u = \lambda_i u$ for $i=1,\ldots,T$.
 then,
 $A_i^p u = \lambda_i^p u$.
 Hence,
 \begin{equation*}
   M_p^p(A_1 , \ldots , A_T)u = \left( {\frac  {1}{T}} \sum_{{i=1}}^{T} A_i^p \right) u =  \left( {\frac  {1}{T}} \sum_{{i=1}}^{T} \lambda_i^p \right)u = m_p^p(\lambda_1 , \ldots , \lambda_T)u
 \end{equation*}
Thus $u$ is an eigenvector of $M_p(A_1 , \ldots , A_T)$ with eigenvalue $m_p(\lambda_1 , \ldots , \lambda_T)$. 
\end{proof}
}


\subsection{Matrix power means for multilayer spectral clustering}\label{subsec:matrix_power_means_for_multilayer_graph_clustering}
We consider the multilayer graph $\multiLayerGraph{G}=(G^{(1)},\ldots,G^{(T)})$ and 
define the \textbf{power mean Laplacian} $\PowerMeanLaplacianCommand{p}$ of $\multiLayerGraph{G}$ as
\begin{equation}\label{eq:powerMeanOperator}
 \PowerMeanLaplacian= M_p(L_{\sym}^{(1)}, \ldots, L_{\sym}^{(T)})
\end{equation}
where $L_{\sym}^{(t)}$ is the normalized Laplacian of the  graph $G^{(t)}$. Note that Definition~\ref{definition:MatrixPowerMean} of the matrix power mean $M_p(A_1 , \ldots , A_T)$ requires $A_1 , \ldots , A_T$ to be positive definite. As the normalized Laplacian is positive semi-definite, in the following, \pedro{for $p\leq 0$}
we add  to $L_{\sym}^{(t)}$ in Equation~\eqref{eq:powerMeanOperator} a small diagonal shift which ensures positive definiteness, that is we consider 
$L_{\sym}^{(t)} + \varepsilon I$ 
throughout the paper. 
\pedro{
For all numerical experiments we set $\epsilon\!=\!\log(1\!+\!\abs{p})$ for $p\!<\!0$ and $\epsilon=10^{-6}$ for $p=0$.
}
Abusing notation slightly, we always mean the shifted versions in the following, unless the shift is explicitly stated.

Similar to spectral clustering for a single graph, we propose Alg.~\ref{alg:powerMean} for the spectral clustering of multilayer graphs based on the matrix power mean of Laplacians. 
\begin{algorithm}[t]
  \DontPrintSemicolon
	\caption{Spectral clustering with $\PowerMeanLaplacian$ on multilayer networks }\label{alg:powerMean}
	\KwIn{Symmetric matrices $W^{(1)},\ldots,W^{(T)}$, number $k$ of clusters to construct. }
	\KwOut{Clusters $C_1,\ldots,C_k$.}
	 Compute eigenvectors $\u_1, \ldots,\u_k$ corresponding to the $k$ smallest eigenvalues of~$\PowerMeanLaplacian$.\;
	 Set $U=(\u_1, \ldots,\u_k)$ and cluster the rows of $U$ with $k$-means into clusters $C_1,\ldots,C_k$.	\;
\end{algorithm}
As in standard spectral clustering, see \cite{Luxburg:2007:tutorial}, our Algorithm~\ref{alg:powerMean} uses the eigenvectors corresponding to the $k$ smallest eigenvalues of the power mean Laplacian $\PowerMeanLaplacian$.
Thus  the relative ordering of the eigenvalues of $\PowerMeanLaplacian$ is of utmost importance.
By Lemma~\ref{lemma:eigenvalues_and_eigenvectors_of_generalized_mean_V2} we know that
if $A_i \u = \lambda(A_i)\u$, for $i=1,\ldots,n$,
then the corresponding eigenvalue of the matrix power mean is $m_p\left(\lambda(A_1), \ldots, \lambda(A_T) \right)$.
Hence, the ordering of eigenvalues strongly depends on the choice of the parameter $p$.
In the next Section we study the effect of the parameter $p$ on the ordering of the eigenvectors of $\PowerMeanLaplacian$ for multilayer graphs following the stochastic block model.

\section{Stochastic block model on multilayer graphs}\label{sec:stochastic_block_model}
In this Section we present an analysis of the eigenvectors and eigenvalues of the power mean Laplacian under the Stochastic Block Model (SBM) for multilayer graphs. 
The SBM is a widespread random graph model for single-layer networks having a prescribed clustering structure \cite{rohe2011spectral}.
Studies of community detection for multilayer networks following the SBM can be found 
in~\cite{Han:2015:CED:3045118.3045279,heimlicher2012community,jog:2015a,xu:2014a,xu:2017a,NIPS2016_6196}.

In order to grasp how different methods identify communities in multilayer graphs following the SBM we will analyze three different settings.
In the first setting all layers follow the same node partition (see f.i.~\cite{Han:2015:CED:3045118.3045279}) and we study the robustness of the spectrum of the power mean Laplacian when the first
layer is informative and the other layers are noise or even contain contradicting information.  In the second setting we consider the
particularly interesting situation where multilayer-clustering is superior over each individual clustering. More specifically, we consider
the case where we are searching for three clusters but each layer contains only information about one of them and only considering
all of the layers together reveals the information about the underlying cluster structure.
In a third setting we go beyond the standard SBM and consider the case where we have a graph partition for each layer, but this partition changes from layer to layer according to a generative model (see f.i.\cite{Bazzi:2016:Community}). However, for the last setting we only provide an empirical study, whereas for the first two settings we analyze the spectrum also analytically. 
For brevity, all the proofs are moved to the Appendix.


In the following we denote by  $\mathcal{C}_1, \ldots, \mathcal{C}_k$ the ground truth clusters that we aim to recover. All the $\mathcal C_i$ are assumed to have the  same size $\abs{\mathcal{C}}$. Calligraphic letters are used for the expected matrices in the SBM. In particular, for a layer $G^{(t)}$ we denote by $\mathcal W^{(t)}$ its expected adjacency matrix, by $\mathcal D^{(t)} = \mathrm{diag}(\mathcal W^{(t)}\one)$ the exptected degree matrix and by $\mathcal L_\sym^{(t)}=I- (\mathcal D^{(t)})^{-1/2}\mathcal W^{(t)} (\mathcal D^{(t)})^{-1/2}$ the expected normalized Laplacian.

\subsection{Case 1: Robustness to noise where all layers have the same cluster structure}\label{sec:sbm_setting1}
The case where all layers follow a given node partition is a natural extension of the mono-layer SBM  to the multilayer setting.
This is done by having different edge probabilities for each layer~\cite{Han:2015:CED:3045118.3045279}, while fixing the same node partition in all layers. We denote by $\pp^{(t)}$ (resp.\ $\qp^{(t)}$) the probability that there exists an edge in layer $G^{(t)}$ between nodes that belong to the same (resp.\ different) clusters.
Then
$\mathcal{W}^{(t)}_{ij}\!\! = \!\pp^{(t)}$ if $v_i,v_j$ belong to the same cluster and $\mathcal{W}^{(t)}_{ij}\!\! =\! \qp^{(t)}$ if $v_i,v_j$ belong to different clusters.
}
Consider the following $k$~vectors:
%
\begin{itemize}[topsep=-3pt,leftmargin=*]\setlength\itemsep{-5pt}
 \item[] \centering $\boldsymbol \chi_1  = \one, \qquad \boldsymbol \chi_i = (k-1)\one_{\mathcal{C}_i}-\one_{\overline{\mathcal{ C}_i}}$\, . 
\end{itemize}
The use of $k$-means on the embedding induced by the vectors $\{\boldsymbol \chi_i\}_{i=1}^{k}$ identifies the ground truth communities $\{\mathcal C_i\}_{i=1}^k$. It turns out that in expectation $\{\boldsymbol \chi_i\}_{i=1}^{k}$ are eigenvectors
of the power mean Laplacian $\PowerMeanLaplacianCommand{p}$. We look for conditions so that they correspond to the 
$k$ smallest eigenvalues as this implies that our spectral clustering Algorithm~\ref{alg:powerMean} recovers the
ground truth.

Before addressing the general case, we discuss the case of two layers. 
For this case we want to illustrate the effect of the power mean by simply studying the extreme limit cases
$$
	\PowerMeanLaplacianInExpectationCommand{\infty}:=\lim_{p\to\infty} \PowerMeanLaplacianInExpectationCommand{p} \quad \text{and} \quad  \PowerMeanLaplacianInExpectationCommand{-\infty}:=\lim_{p\to-\infty} \PowerMeanLaplacianInExpectationCommand{p}\, .
$$
where $\PowerMeanLaplacianInExpectationCommand{p} = M_p(\mathcal L^{(1)}_\sym , \mathcal L^{(2)}_\sym)$. The next Lemma shows that $\PowerMeanLaplacianInExpectationCommand{\infty}$ and $\PowerMeanLaplacianInExpectationCommand{-\infty}$
are related to the logical operators \texttt{AND} and \texttt{OR}, respectively,
in the sense that in expectation $\PowerMeanLaplacianInExpectationCommand{\infty}$
recovers the clusters if and only if $G^{(1)}$ \textbf{and} $G^{(2)}$ have both clustering structure, whereas in expectation
$\PowerMeanLaplacianInExpectationCommand{-\infty}$ recovers the clusters if and only if $G^{(1)}$ \textbf{or} $G^{(2)}$ has clustering structure.
\begin{lemma}\label{lemma:L_limitCases2GraphsOnlyLaplacians}
Let $\PowerMeanLaplacianInExpectationCommand{p} = M_p(\mathcal L^{(1)}_\sym , \mathcal L^{(2)}_\sym)$.
\begin{itemize}[topsep=-3pt,leftmargin=*]
 \item $\{ \boldsymbol \chi_i \}_{i=1}^{k}$ correspond to the $k$ smallest eigenvalues of $\mathcal{L}_{\infty}$ if and only if 
 $\pp^{(1)} > \qp^{(1)}$ \textbf{and} $\pp^{(2)} > \qp^{(2)}$.
  \item $\{ \boldsymbol \chi_i \}_{i=1}^{k}$ correspond to the $k$ smallest eigenvalues of $\mathcal{L}_{-\infty}$ if and only if 
 $\pp^{(1)} > \qp^{(1)}$ \textbf{or} $\pp^{(2)} > \qp^{(2)}$.
\end{itemize}
\end{lemma}
\myComment{ 
\begin{proof}
 Let $T=2$ and $r=2$. The result follows directly from Theorem~\ref{theorem:limitCases}.
 \end{proof}
}
\begin{figure}[t]
\begin{subfigure}[]{0.9\linewidth}
\centering
 \includegraphics[width=1.1\linewidth, clip,trim=150 160 160 435]{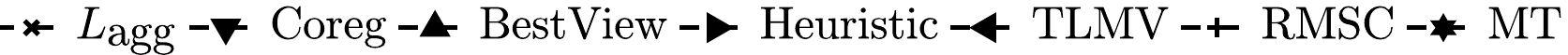}

\end{subfigure}%

\begin{subfigure}[]{1\linewidth}\includegraphics[width=1\linewidth, clip,trim=168 159 180 435]{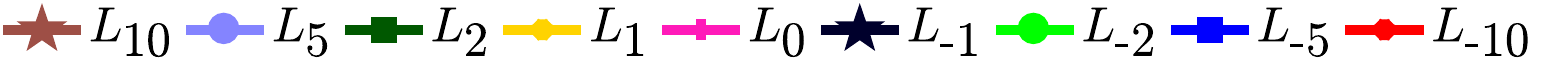}
\end{subfigure}%

\begin{subfigure}[]{0.53\linewidth}
 \includegraphics[width=1\linewidth, clip,trim=120 30 125 20]{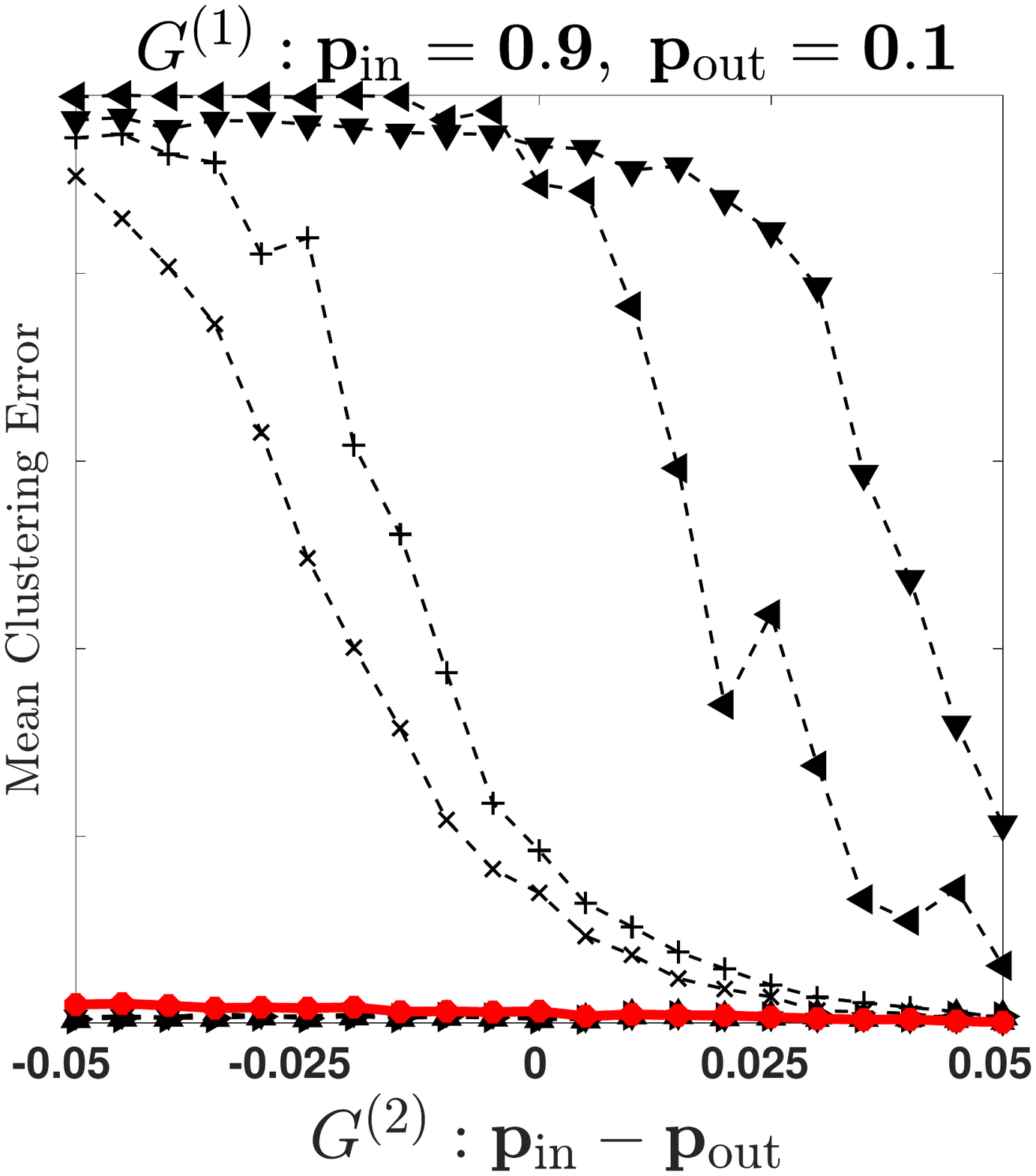}\hspace*{\fill}\hspace{-1.5em}
\caption{$L_p = M_p(L_{\textrm{sym}}^{(1)},L_{\textrm{sym}}^{(2)})$}
\label{fig:SBM2ClustersLL}
\end{subfigure}%
\begin{subfigure}[]{0.53\linewidth}
\includegraphics[width=0.97\linewidth, clip,trim=160 30 100 20]
{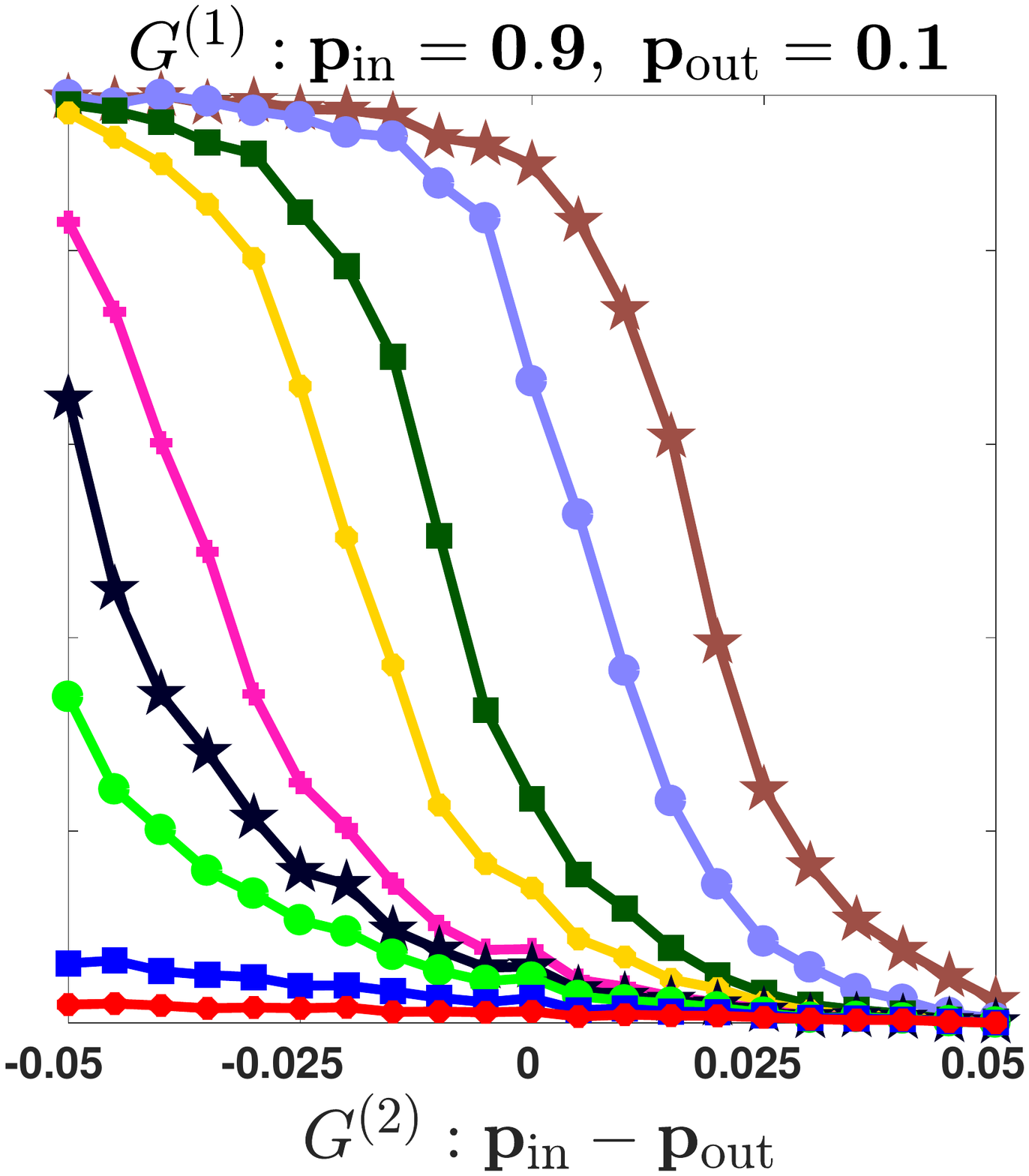}\hspace*{\fill}\hspace{-1.0em}
\caption{$L_p = M_p(L_{\textrm{sym}}^{(1)},L_{\textrm{sym}}^{(2)})$}
\label{fig:SBM2ClustersQQ}
\end{subfigure}%
\caption{
Mean Clustering Error under the SBM with two clusters.
First layer $G^{(1)}$ is \textit{assortative} and ${L_p = M_p(L_{\textrm{sym}}^{(1)},L_{\textrm{sym}}^{(2)})}$. 
Second layer $G^{(2)}$ transitions from disassortative to assortative.
Fig.~\ref{fig:SBM2ClustersLL}:  Comparison of $L_{-10}$ with state of art.
Fig.~\ref{fig:SBM2ClustersQQ}: Performance of $L_{p}$ with ${p\in\{0,\pm1,\pm2,\pm5,\pm10\}}$.
}
\label{fig:SBM2Clusters}
\end{figure}


%
%
The following theorem gives general conditions on the recovery of the ground truth clusters in dependency on $p$ and the size of the shift in  $\L{p}$, see Section \ref{subsec:matrix_power_means_for_multilayer_graph_clustering}. Note that, in analogy with Lemma \ref{lemma:L_limitCases2GraphsOnlyLaplacians}, as $p\rightarrow -\infty$ 
the recovery of the ground truth clusters is achieved
if 
at least
one of the layers is informative, whereas if 
$p\rightarrow \infty$ all of them have to be informative in order to recover the ground truth.

\begin{figure*}[t]
\centering
\frame{
\includegraphics[width=0.5\linewidth, clip,trim=150 160 160 435]{figure1a.pdf}
 \includegraphics[width=0.48\linewidth, clip,trim=168 159 180 435]{figure1b.pdf}
}
\vskip.6em

$\,$
\hfill
\begin{subfigure}[]{0.245\linewidth}
 \includegraphics[width=1\linewidth, clip,trim=120 40 100 40]{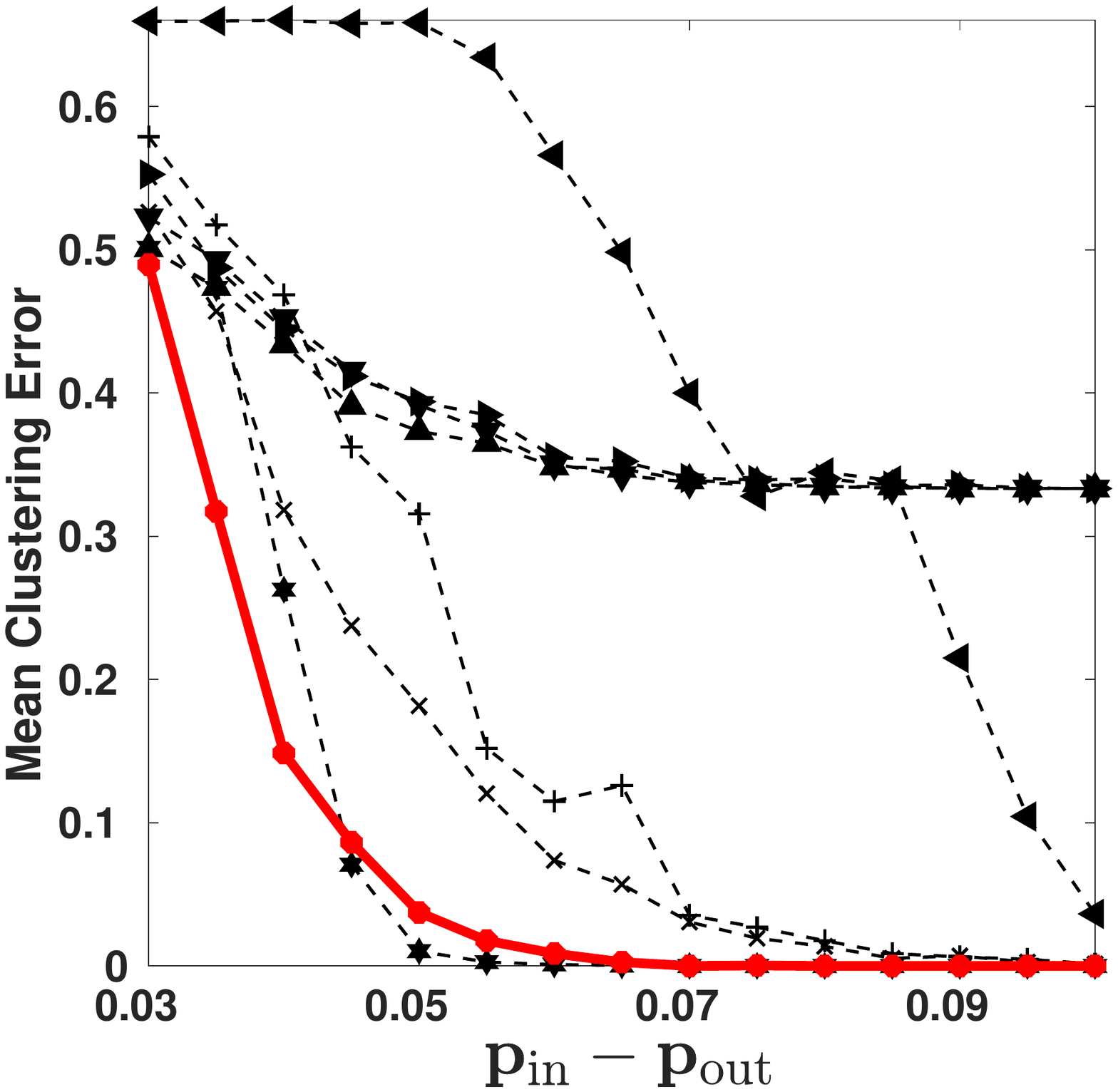}\hspace*{\fill}
\caption{}
\label{subfig:3LayersErrorA}
\end{subfigure}%
\hfill
\begin{subfigure}[]{0.245\linewidth}
 \includegraphics[width=1\linewidth, clip,trim=120 40 100 40]{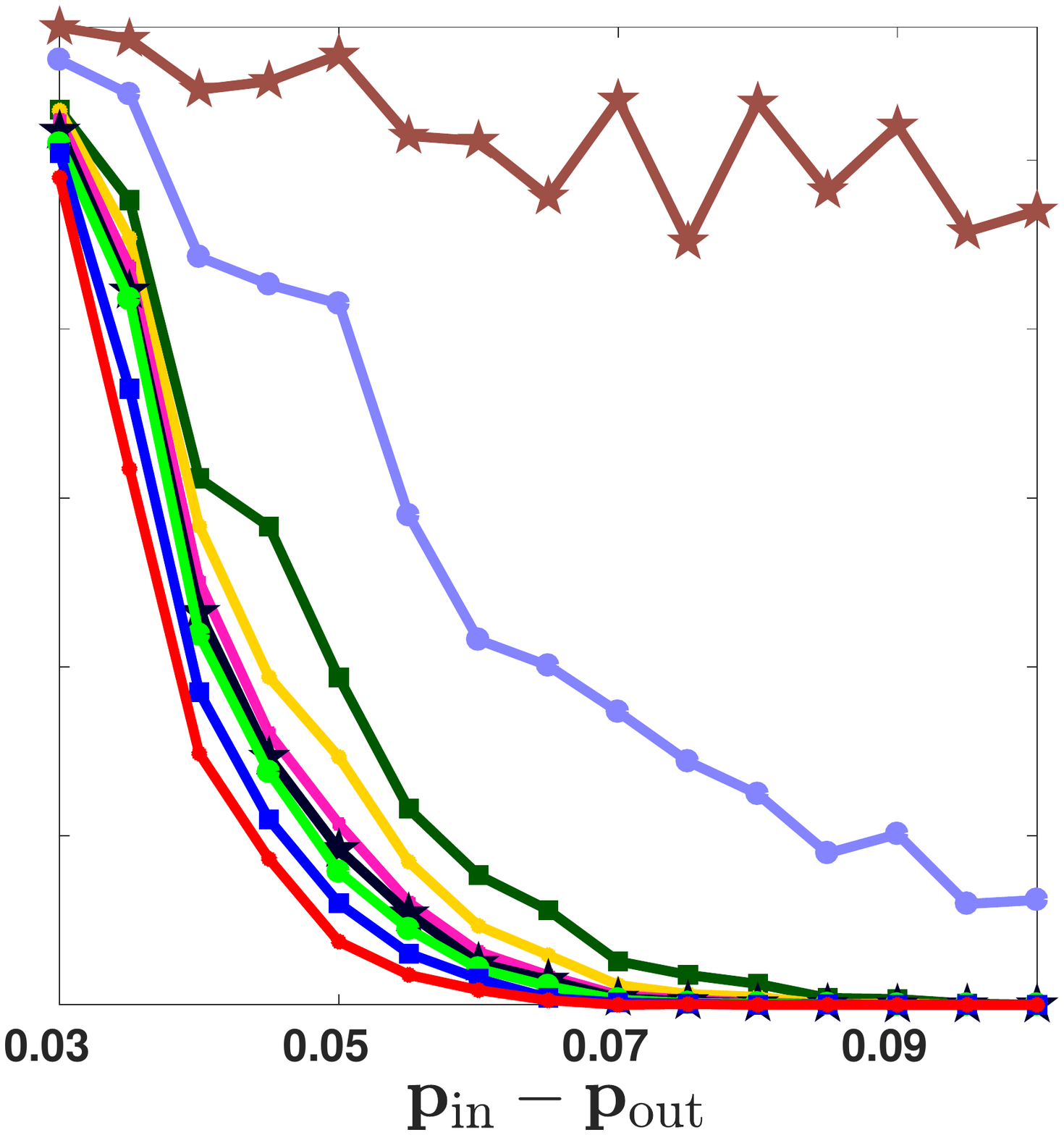}\hspace*{\fill}
\caption{}
\label{subfig:3LayersError}
\end{subfigure}%
\hfill
\begin{subfigure}[]{0.245\linewidth}
\includegraphics[width=1\linewidth, clip,trim=120 40 100 40]{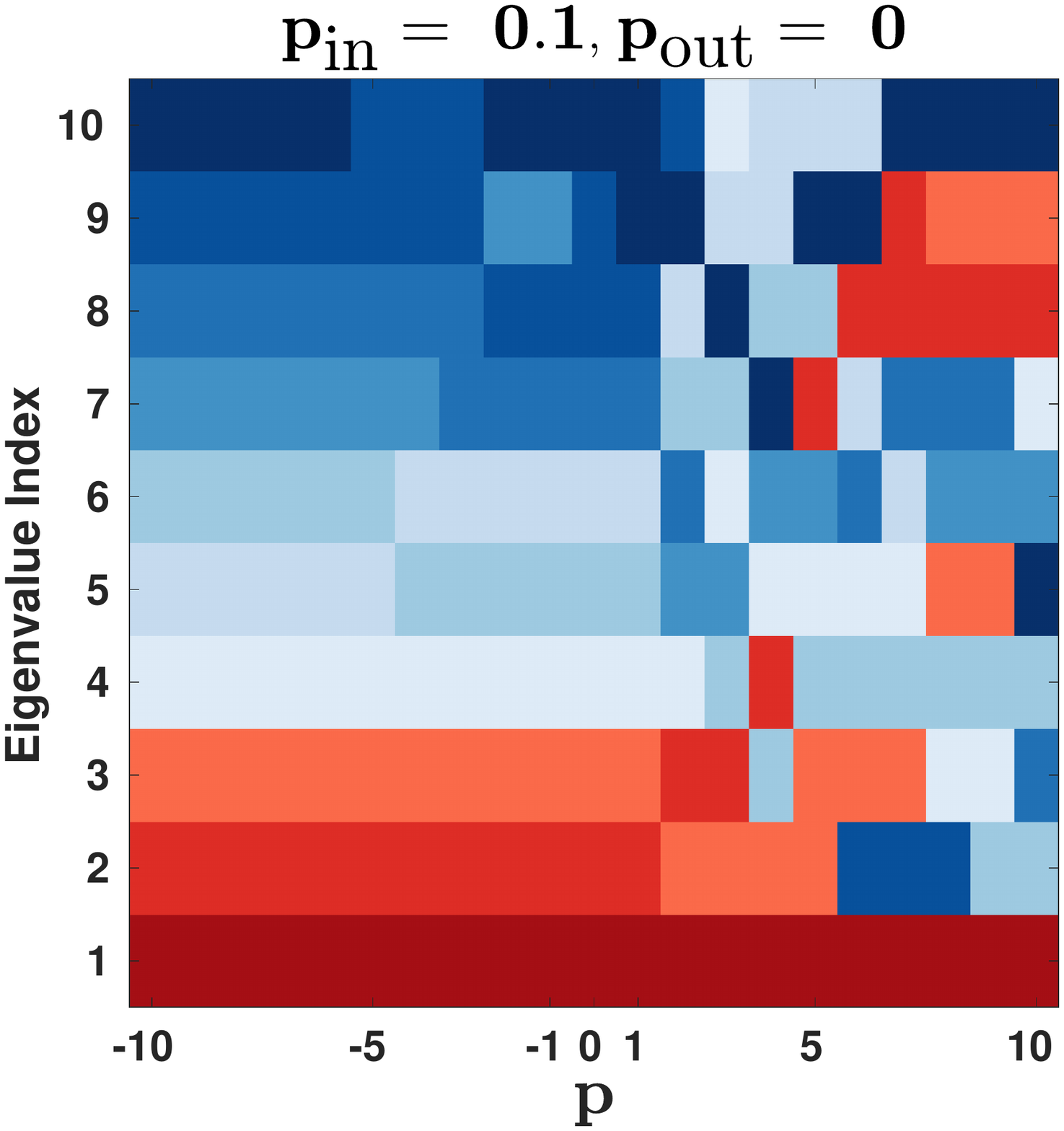}\hspace*{\fill}
\caption{}
\label{subfig:3LayersSpaguetti}
\end{subfigure}%
\hfill
\begin{subfigure}[]{0.245\linewidth}
\includegraphics[width=1\linewidth, clip,trim=120 40 100 40]
{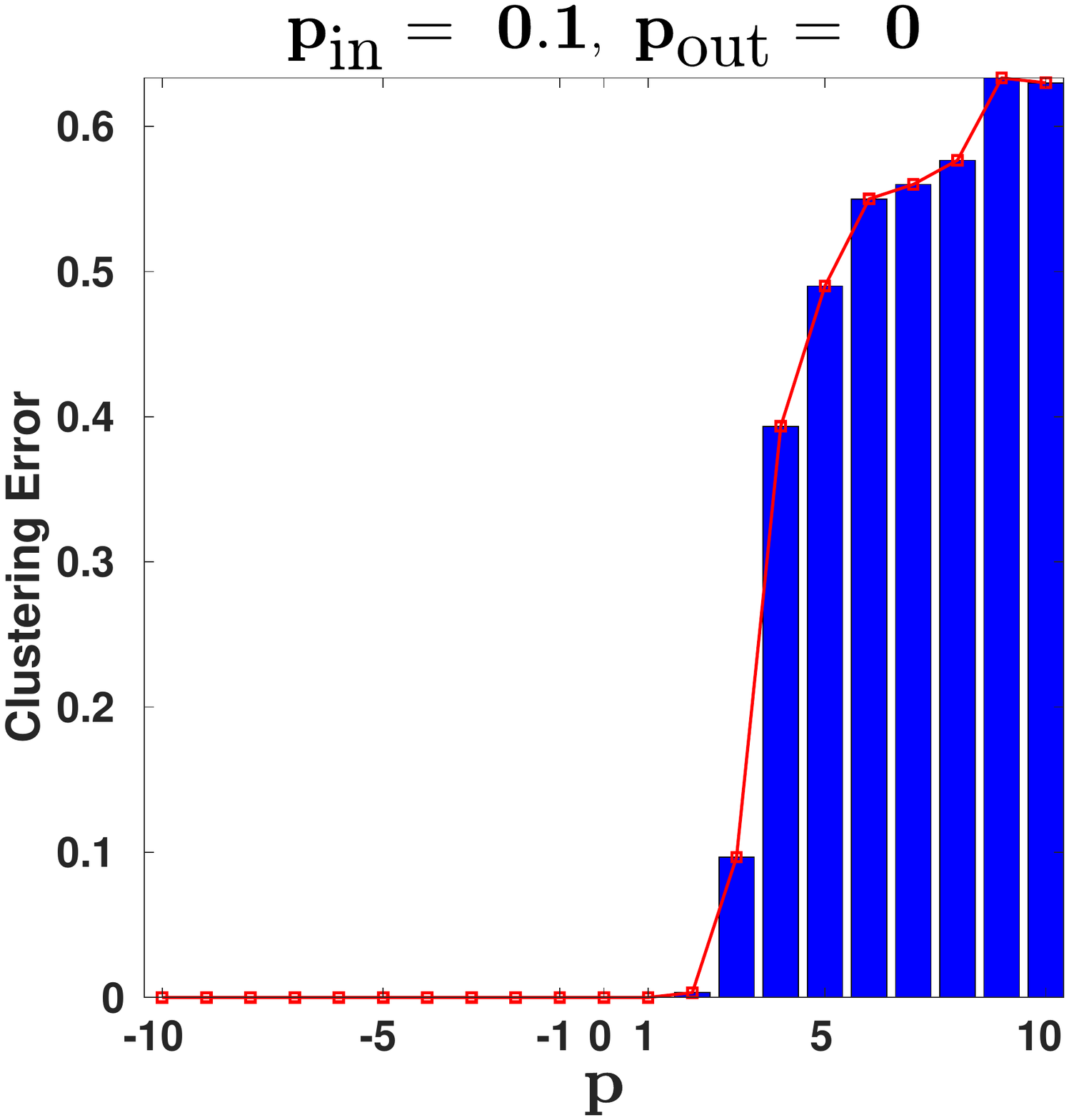}
\caption{}
\label{subfig:3LayersBars}
\end{subfigure}%
\hfill
$\,$
\caption{
SBM experiments with three layers. Each layer is informative with respect to one cluster. 
\ref{subfig:3LayersErrorA}: Comparison of $L_{-10}$ with state of art.
\ref{subfig:3LayersError}: Performance of $L_{p}$ with ${p\in\{0,\pm1,\pm2,\pm5,\pm10\}}$.
\ref{subfig:3LayersSpaguetti}: Eigenvalue ordering of power mean Laplacian $L_p$ across different powers. The ordering clearly changes for powers $p\geq2$, inducing non-informative eigenvectors to the bottom of the spectrum.
\ref{subfig:3LayersBars}: Clustering error of the power mean Laplacian $L_p$. Clustering error increases with $p\geq2$, as suggested by ordering changes depicted in ~\ref{subfig:3LayersSpaguetti}. 
}
\label{fig:3Layers}

\end{figure*}



\begin{thm}\label{thm:MatrixPowerMeanLaplacian}
Let $p\!\in\![-\infty,\infty]$, then $\bchi_1,\ldots,\bchi_k$ correspond to the $k$-smallest eigenvalues of $\L{p}$ if and only if $m_p(\boldsymbol{\mu}+\epsilon\ones)\!<\!1+\epsilon$, where $\boldsymbol{\mu}=(1-\rho_1,\ldots,1-\rho_T)$, and 
$\rho_t\!=\!(\pp^{(t)}\!-\!\qp^{(t)})/(\pp^{(t)}\!+\!(k\!-\!1)\qp^{(t)})$.

In particular, for $p\to \pm\infty$, we have
\begin{enumerate}[topsep=-3pt,leftmargin=*,resume]\setlength\itemsep{-3pt}
\item $\bchi_1,\ldots,\bchi_k$ correspond to the $k$-smallest eigenvalues of $\L{\infty}$ if and only if 
all layers are informative, i.e.
$\pp^{(t)} > \qp^{(t)}$ holds for all $t\in\{1,\ldots,T\}$.
\item $\bchi_1,\ldots,\bchi_k$ correspond to the $k$-smallest eigenvalues of $\L{-\infty}$ if and only if 
there is at least one informative layer, i.e.
 there exists a $t\in\{1,\ldots,T\}$ such that $\pp^{(t)} > \qp^{(t)}$. 
\end{enumerate}
\end{thm}


\myComment{
\begin{proof}
The proof is in Section~\ref{proof:MatrixPowerMeanLaplacian}.
\end{proof}
}
Theorem~\ref{thm:MatrixPowerMeanLaplacian} shows that the informative eigenvectors of $\mathcal L_p$ are at the bottom of the spectrum if and only if the scalar power mean of the corresponding eigenvalues is small enough.
Since the scalar power mean is monotonically decreasing with respect to $p$, 
this explains why the limit case $p\to\infty$ is more restrictive than $p\to \!-\infty$. 
The corollary below shows that
the coverage of parameter settings in the SBM for which one recovers the ground truth becomes smaller as $p$~grows.

\begin{corollary}\label{corollary:contention}
Let $q\leq p$.
If $\bchi_1,\ldots,\bchi_k$ correspond to the $k$-smallest eigenvalues of $\L{p}$, then
$\bchi_1,\ldots,\bchi_k$ correspond to the $k$-smallest eigenvalues of $\L{q}$.
\end{corollary}

The previous results hold in expectation.  The following experiments show that these findings generalize to the case 
where one samples from the SBM. 
In Fig.~\ref{fig:SBM2Clusters} we present experiments on sparse sampled multilayer graphs from the SBM.
We consider two clusters of size $\abs{\mathcal C} = 100$ and show the mean of clustering error of $50$ runs. 
We evaluate the power mean Laplacian $\PowerMeanLaplacianCommand{p}$ with 
$p\in\{0,\pm1,\pm2,\pm5,\pm10\}$ and compare with other methods described in Section~\ref{sec:experiment}.

In Fig.~\ref{fig:SBM2Clusters} we fix the first layer $G^{(1)}$ to be strongly assortative and let the second layer $G^{(2)}$
run from a disassortative to an assortative configuration. 
In Fig.\ref{fig:SBM2ClustersLL} we can see that the power mean Laplacian $L_{-10}$ returns the smallest clustering error,
together with the multitensor method, the best single view and the heuristic approach across all parameter settings. The latter two work  well by construction in this setting. However, we will see that they fail for the second setting we consider next. All the other competing methods
fail as the second graph $G^{(2)}$ becomes non-informative resp.\ even violates the assumption to be assortative.
In Fig.~\ref{fig:SBM2ClustersQQ} we can see that the smaller the value of $p$, the smaller the clustering
error of the power mean Laplacian $L_p$,
as stated in Corollary~\ref{corollary:contention}.
%


\subsection{Case 2: No layer contains full information on the clustering structure}\label{sec:sbm_setting2}



We consider a multilayer SBM setting where each individual layer contains only information about one of the clusters and only
considering all the layers together reveals the complete cluster structure. For this particular instance, all power mean Laplacians
$\L{p}$ allow to recover the ground truth for any non-zero integer $p$.


For the sake of simplicity, we limit ourselves to the case of three layers and three clusters, showing an assortative behavior in expectation. Let the expected adjacency matrix $\mathcal W^{(t)}$ of layer $G^{(t)}$ be defined~by
\begin{equation}\label{eq:3Clusters}
 \mathcal W^{(t)}_{i,j} = 
\begin{cases} 
      \pp, & \quad v_i,v_j\in \mathcal C_t \text{ or }  v_i,v_j\in \overline{\mathcal C_t}\\
      \qp, & \quad \textrm{else} \\
   \end{cases} 
\end{equation}
for ${t=1,2,3}$. Note that, up to a node relabeling, the three expected adjacency matrices have the form
$$
\underbrace{\left(
\begin{array}{c|c|c}
    \grayBlock & \0         & \0         \\  \hline
    \0         & \grayBlock & \grayBlock \\  \hline
    \0         & \grayBlock & \grayBlock \\
\end{array}
\right)}_{\mathcal W^{(1)}}, \quad
\underbrace{\left(
\begin{array}{c|c|c}
    \grayBlock & \0 & \grayBlock         \\  \hline
    \0 & \grayBlock & \0         \\  \hline
    \grayBlock        & \0         & \grayBlock \\
\end{array}
\right)}_{\mathcal W^{(2)}}, \quad
\underbrace{\left(
\begin{array}{c|c|c}
    \grayBlock & \grayBlock & \0         \\  \hline
    \grayBlock & \grayBlock & \0         \\  \hline
    \0         & \0         & \grayBlock \\
\end{array}
\right)}_{\mathcal W^{(3)}}\, ,
$$
where
each (block) row and column corresponds to a cluster $\mathcal C_i$ and gray blocks correspond to nodes whose probability of connections is $\pp$, whereas white blocks correspond to nodes whose probability of connections is $\qp$. 
Let us assume an assortative behavior on all the layers, that is $\pp>\qp$. In this case spectral clustering applied on a single layer ${\mathcal W^{(t)}}$ would return cluster $\mathcal C_t$ and 
a random partition of the complement,  
failing to recover the ground truth clustering $\mathcal C_1,\mathcal C_2,\mathcal C_3$. 
This is shown in the following Theorem.
\begin{thm}\label{thm:spectrum_L}
If $\pp>\qp$, then for any $t=1,2,3$, there exist scalars $\alpha>0$ and $\beta>0$ such that the eigenvectors of $\mathcal L_\sym^{(t)}$ corresponding to the two smallest eigenvalues are
$$
\boldsymbol \chi_1 = \alpha \one_{\mathcal C_t}+\one_{\overline{\mathcal C_t}} \quad \text{and} \quad
\boldsymbol \chi_2 =-\beta \one_{\mathcal C_t}+\one_{\overline{\mathcal C_t}}
$$
whereas any vector orthogonal to both $\boldsymbol \chi_1$ and $\boldsymbol \chi_2$ is an eigenvector for the third smallest eigenvalue. 
\end{thm}
%
%
%
On the other hand, it turns out that the power mean Laplacian $L_p$ is able to merge the information of each layer, obtaining the ground truth clustering, for all integer powers different from zero.
This is formally stated in the following. 

\begin{table*}\scriptsize
\parbox{.45\linewidth}{
\centering                      
\begin{tabular}{|c|c|c|c|c|c|c|}   
\multicolumn{7}{c}{\,\,\,\,\,$\bm{\tilde p}$} \\
\hline        
  & 0.5 & 0.6 & 0.7 & 0.8 & 0.9 & 1.0 \\  
\specialrule{1.5pt}{.0pt}{0pt} 
$L_\text{agg}$ & 0.3 & 1.3 & 3.0 & 8.0 & 22.3 & 100.0 \\     
\text{Coreg} & 0.3 & 0.0 & 0.3 & 0.0 & 0.0 & 64.7 \\        
\text{BestView} & 9.7 & 1.0 & 0.3 & 0.0 & 0.7 & 77.3 \\     
\text{Heuristic} & 0.0 & 0.0 & 0.0 & 0.0 & 0.3 & 59.3 \\    
\text{TLMV} & 0.7 & 0.7 & 4.0 & 6.0 & 24.7 & 100.0 \\       
\text{RMSC} & 1.0 & 1.7 & 4.0 & 7.0 & 19.7 & 100.0 \\       
\text{MT} & 1.3 & 0.3 & 0.7 & 3.0 & 17.0 & 100.0 \\
\specialrule{0.75pt}{.0pt}{0pt}  
$L_{10}$ & 0.0 & 0.0 & 0.0 & 0.0 & 1.0 & 100.0 \\          
$L_{5}$ & 0.0 & 0.0 & 0.0 & 0.0 & 5.0 & 100.0 \\           
$L_{2}$ & 0.0 & 0.0 & 0.3 & 2.3 & 18.3 & 100.0 \\          
$L_{1}$ & 1.0 & 1.0 & 3.0 & 7.0 & 30.3 & 100.0 \\          
$L_{0}$ & 4.3 & 4.3 & 9.7 & 15.3 & 38.3 & 100.0 \\         
$L_{-1}$ & 6.7 & 7.7 & 15.7 & 16.3 & 42.3 & 100.0 \\       
$L_{-2}$ & 8.0 & 13.0 & 20.3 & 20.7 & 42.7 & 100.0 \\      
$L_{-5}$ & 22.3 & 23.0 & 36.3 & 37.7 & 50.0 & 100.0 \\     
$L_{-10}$ & \textbf{69.0} & \textbf{76.3} & \textbf{68.0} & \textbf{67.3} & \textbf{59.7} & 100.0 \\   
\hline                                                  
\end{tabular}                           
}
\hspace{10mm}
\parbox{.45\linewidth}{
\centering
\begin{tabular}{|c|c|c|c|c|c|c|}   
\multicolumn{7}{c}{\,\,\,\,\,$\bm{\mu}$} \\ 
\hline                                                       
  & 0.0 & 0.1 & 0.2 & 0.3 & 0.4 & 0.5 \\
\specialrule{1.5pt}{.0pt}{0pt}  
$L_\text{agg}$ & 24.7 & 21.7 & 21.3 & 21.7 & 24.3 & 21.3 \\     
\text{Coreg} & 16.7 & 16.7 & 13.3 & 11.7 & 6.0 & 1.0 \\        
\text{BestView} & 16.7 & 17.0 & 17.0 & 17.7 & 11.7 & 9.0 \\    
\text{Heuristic} & 16.7 & 16.3 & 15.0 & 9.0 & 2.0 & 0.7 \\     
\text{TLMV} & 25.7 & 24.3 & 21.7 & 23.3 & 21.0 & 20.0 \\       
\text{RMSC} & 26.3 & 22.0 & 23.0 & 21.7 & 20.3 & 20.0 \\       
\text{MT} & 19.7 & 19.7 & 21.0 & 20.7 & 20.7 & 20.7 \\
\specialrule{0.75pt}{.0pt}{0pt}  
$L_{10}$ & 16.7 & 17.3 & 17.0 & 16.7 & 16.7 & 16.7 \\         
$L_{5}$ & 17.0 & 18.0 & 17.3 & 17.7 & 18.0 & 17.0 \\          
$L_{2}$ & 23.0 & 21.3 & 19.3 & 19.0 & 20.3 & 18.0 \\          
$L_{1}$ & 26.3 & 25.3 & 24.0 & 23.0 & 22.3 & 21.3 \\          
$L_{0}$ & 33.3 & 30.3 & 28.7 & 28.0 & 28.0 & 23.7 \\          
$L_{-1}$ & 36.3 & 33.0 & 33.3 & 32.0 & 29.0 & 25.0 \\         
$L_{-2}$ & 37.3 & 36.3 & 36.7 & 34.0 & 31.3 & 29.0 \\         
$L_{-5}$ & 48.0 & 45.0 & 49.0 & 44.3 & 43.0 & 40.0 \\         
$L_{-10}$ & \textbf{71.7} & \textbf{72.3} & \textbf{72.7} & \textbf{74.7} & \textbf{76.3} & \textbf{72.7} \\
\hline                    
\end{tabular}             
}
\caption{
Percentage of cases where the minimum clustering error is achieved by different methods.
Left: Columns correspond to a fixed value of $\tilde p$ and we aggregate over $\mu\in\{0.0,0.1,0.2,0.3,0.4,0.5\}$.
Left: Columns correspond to a fixed value of $\mu$     and we aggregate over  $\tilde p\in\{0.5,0.6,0.7,0.8,0.9,1.0\}$.
}
\label{table:DCSBM}
\end{table*}
%
%
%
%
\begin{thm}\label{specmainthm}
Let $\pp>\qp$ and for $\varepsilon>0$ define \vskip-1em
$$
\tilde{\mathcal L}_{\mathrm{sym}}^{(t)} = \mathcal L_\sym^{(t)} + \varepsilon I, \quad t=1,2,3.
$$
Then 
the eigenvectors of 
$\mathcal L_p=M_p(\tilde{\mathcal L}_{\sym}^{(1)},\tilde{\mathcal L}_{\sym}^{(2)},\tilde{\mathcal L}_{\sym}^{(3)})$ 

corresponding to its three smallest eigenvalues are
$$\boldsymbol \chi_1=\one, \quad \boldsymbol \chi_2=\one_{\mathcal C_2}-\one_{\mathcal C_1}, \quad \text{and}\quad \boldsymbol \chi_3=\one_{\mathcal C_3}-\one_{\mathcal C_1} $$
for any nonzero integer $p$.
\end{thm} 


The proof of  Theorem \ref{specmainthm} is more delicate than the one of Theorem \ref{thm:MatrixPowerMeanLaplacian}, as it involves the addition of powers of matrices that do not have the same eigenvectors. 

Note that  Theorem~\ref{specmainthm} does not distinguish the behavior for distinct values of $p$. In expectation all nonzero integer
values of $p$ work the same. This is different to Theorem~\ref{thm:MatrixPowerMeanLaplacian}, where the choice of $p$ had a relevant influence on the eigenvector embedding even in expectation. However, we see in the  experiments on graphs sampled from the SBM (Figure \ref{fig:3Layers}) that the choice of $p$ has indeed a significant influence on the performance even though they are the same
in expectation. This suggests that the smaller $p$, the smaller the variance in the difference to the expected behavior in the SBM. We leave this as an open problem if such a dependency can be shown analytically.

In Figs.~\ref{subfig:3LayersErrorA} and~\ref{subfig:3LayersError} we present the mean clustering error out of ten runs.
In Fig.~\ref{subfig:3LayersErrorA} one can see that BestView and Heuristic, which rely on clusterings determined by single views, return high clustering errors which correspond to the identification of only a single cluster. The result of Theorem~\ref{specmainthm} explains this failure.
The reason for the increasing clustering error with $p$ can be seen
in Fig.~\ref{subfig:3LayersSpaguetti} where we analyze how the ordering of eigenvectors changes for different values of $p$.
We can see that for negative powers, the informative eigenvectors belong to the bottom three eigenvalues (denoted in red).
For the cases where $p\geq\! 2$ the ordering changes, pushing non-informative eigenvectors to the bottom of the spectrum and thus
resulting into a high clustering error, as seen in Fig.~\ref{subfig:3LayersBars}. However, we conclude that also for this second case a
strongly negative power mean Laplacian as $L_{-10}$ works~best.
%
%
%
%
%
%
%
%
%
%
%
%
%
%
%
%
\makeatletter
\newcommand{\removelatexerror}{\let\@latex@error\@gobble}
\makeatother
\removelatexerror
\begin{figure*}
    \makeatletter
    \def\@captype{algocf}
    \makeatother
    \begin{minipage}{0.45\textwidth}
    \vspace{-4pt}
          \begin{algorithm}[H]
          \DontPrintSemicolon
		\caption{\footnotesize{PM applied to $M_p.^{\textcolor{white}{1/2}}$}}\label{alg:PM}
		{\footnotesize
		\KwIn{$\x_0$, $p<0$}
		\KwOut{Eigenpair  $(\lambda, \x)$ of $M_p$}
		\Repeat{tolerance reached}{
			   $\u^{(1)}_k$ $\gets$ $(A_{1})^{p} \x_k$ \;
			   $\vdots$\;
		        $\u^{(T)}_k$ $\gets$ $(A_{T})^{p} \x_k$ \;
			$\y_{k+1}$ $\gets$ ${\frac  {1}{T}} \sum _{{i=1}}^{T}\u^{(i)}_k$\;
			$\x_{k+1}$ $\gets$ $\y_{k+1} / \|\y_{k+1} \|_2$\;
			}
		$\lambda$ $\gets$ $(\x_{k+1}^T \x_k)^{1/p}$,\quad  $\x$ $\gets$ $\x_{k+1}$\;
		}
		\end{algorithm}
    \end{minipage}\hfill
    \begin{minipage}{0.45\textwidth}
         \begin{algorithm}[H]
          \DontPrintSemicolon
		\caption{\footnotesize{PKSM for the computation of $A^p\y$}}\label{alg:kyrlov}
		{\footnotesize
		\KwIn{$\u_0 = \y$, $V_0 = [\,\cdot\,  ], p<0$ }
		\KwOut{$\x=A^p\y$}
		$\v_0$ $\gets$ $\y/\vectornorm{\y}_2$\;
		\For{$s=0,1,2,\dots,n$}{
			$\tilde V_{s+1}$ $\gets$ $[V_s,\v_{s}]$\;
			$V_{s+1}$ $\gets$ Orthogonalize columns of $\tilde V_{s+1}$\;
			$H_{s+1}$ $\gets$ $V_{s+1}^T AV_{s+1}$\;
			$\x_{s+1}$ $\gets$ $V_{s+1} (H_{s+1})^p\e_1 \vectornorm{\y}_2$\;
			\textbf{if} {\it tolerance reached} \textbf{then} {\it break} \;
			$\v_{s+1}$ $\gets$  $A\v_{s}$	\;	
		}
		$\x$ $\gets$ $\x_{s+1}$\;
		}
		\end{algorithm}
    \end{minipage}
\end{figure*}
\begin{figure*}[h]
\floatbox[{\capbeside\thisfloatsetup{capbesideposition={right,top},capbesidewidth=.6\textwidth}}]{figure}[\FBwidth]%
{
\caption{
Mean execution time of 10 runs for the power mean Laplacian $L_p$. 
$L_{-1}(\textrm{ours}), L_{-2}(\textrm{ours}), L_{-5}(\textrm{ours}), L_{-10}(\textrm{ours})$
stands for the power mean Laplacian together with our proposed Power Method (Alg.\ \ref{alg:PM}) based on the Polynomial Krylov Approximation Method (Alg.\ \ref{alg:kyrlov}).
$L_{1}(\textrm{eigs})$ stands for the arithmetic mean Laplacian together with Matlab's \texttt{eigs} function. 
Experiments are performed using one thread.
We generate multilayer graphs with two layers, each with two clusters of same size with parameters $p_{\mathrm{in}} = 0.05$ and $p_{\mathrm{in}} = 0.025$
and graphs of size $\left| V \right|\in\{10000,20000,30000,40000\}$.
}
    \label{fig:timeComparison}}{\includegraphics[width=.3\textwidth, clip,trim=40 35 80 65]{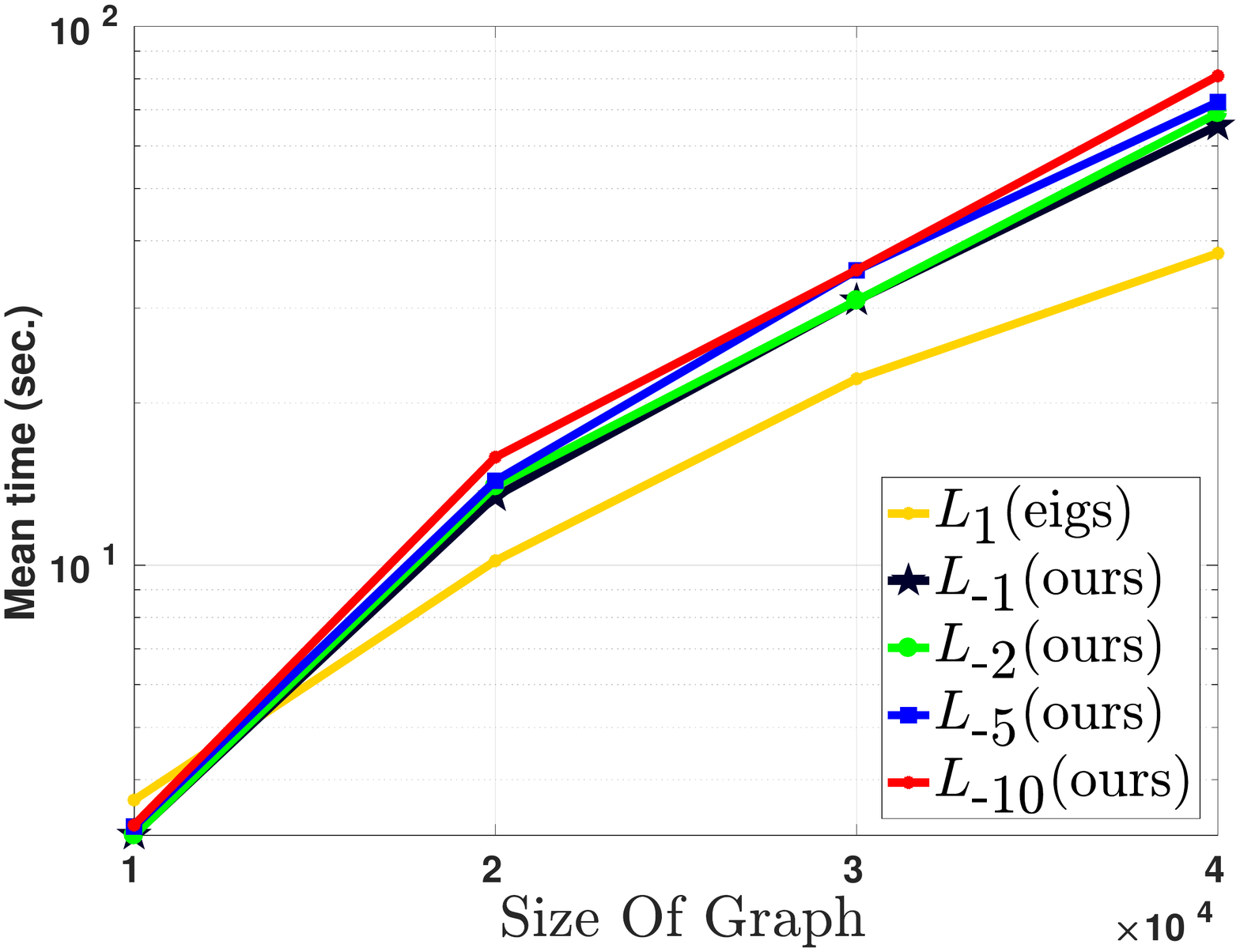}}
\end{figure*}
\subsection{Case 3: Non-consistent partitions between layers}
%
%
We now consider the case where all the layers follow the same node partition (as in Section~\ref{sec:sbm_setting1}),
but the partitions may fluctuate from layer to layer with a certain probability.
We use the multilayer network model introduced in~\cite{Bazzi:2016:Community}.
This generative model considers a graph partition for each layer, allowing the partitions to change from layer to layer according to an interlayer dependency tensor.
For the sake of clarity we consider a one-parameter interlayer dependency tensor with parameter $\tilde p\in [0,1]$ 
(i.e.\ a uniform multiplex network according to the notation used in Section 3.B in~\cite{Bazzi:2016:Community}), 
where for $\tilde p=0$ the partitions between layers are independent, and for $\tilde p=1$  the partitions between layers are identical.
Once the partitions are obtained, edges are generated following a multilayer
degree-corrected SBM (DCSBM in Section 4 of~\cite{Bazzi:2016:Community}), according to a one-parameter affinity matrix with parameter $\mu \in [0,1]$, where for $\mu=0$ all edges are within communities whereas for $\mu=1$ edges are assigned ignoring the community structure.

We choose 
${\tilde p\in\{0.5,0.6,0.7,0.8,0.9,1\}}$ and
${\mu\in\{0.0,0.1,0.2,0.3,0.4,0.5\}}$
and consider all possible combinations of $(\tilde p, \mu)$.
For each pair we count how many times, out of $50$ runs, each method achieves the smallest clustering error.
The remaining parameters of the DCSBM are set as follows:
exponent $\gamma\!\!=\!\!-3$, minimum degree and maximum degree $k_{min}\!=\!k_{max}\!=\!10$,
$\abs{V}\!=\!100$ nodes, $T\!=\!10$ layers and $K\!=\!2$ communities.
As partitions between layers are not necessarily the same, we take the most frequent node assignment among all 10 layers as ground truth clustering.


In Table~\ref{table:DCSBM}, left side, we show the result for fixed values of $\tilde p$ and average over all values of $\mu$. 
On the right table we show the corresponding results for fixed values of $\mu$ and average over all values of $\tilde p$. 
On the left table we can see that for  $\tilde p=1$, where the partition is the same in all layers, all methods recover the clustering,
while, as one would expect, the performance decreases with smaller values of $\tilde p$. 
Further, we note that the performance of the power mean Laplacian improves as $\tilde p$ decreases and $L_{-10}$ again achieves
the best result.
On the right table we see that performance is degrading with larger values of $\mu$. This is expected as for larger values of $\mu$ the  edges inside the clusters are less concentrated. Again the performance of the power mean Laplacian improves as $p$ decreases and $L_{-10}$ performs best.
%
%
%
%
%
%
%
%
%
%
%
%
%
%
%
%
%
%
\begin{table*}[h]\small
\newcolumntype{"}{@{\hskip\tabcolsep\vrule width 1pt\hskip\tabcolsep}}
\setlength\extrarowheight{1pt}
\setlength{\tabcolsep}{3pt}
\centering        
\begin{tabular}{ccccccccc}   
 \specialrule{1.5pt}{.1pt}{2pt} 
                           & \textbf{3Sources}    & \textbf{BBC} & \textbf{BBCS}     & \textbf{Wiki}  & \textbf{UCI} & \textbf{Citeseer} & \textbf{Cora} & \textbf{WebKB}\\
\specialrule{1pt}{.1pt}{0pt} 
\footnotesize{\# vertices} & 169                  & 685          & 544               & 693  	      & 2000         & 3312              & 2708          & 187 \\            
\footnotesize{\# layers}   & 3      		  & 4            & 2                 & 2  	      & 6            & 2                 & 2             & 2   \\ 
\footnotesize{\# classes}  & 6      		  & 5            & 5                 & 10  	      & 10           & 6                 & 7             & 5   \\ 

\specialrule{1pt}{-1pt}{1pt} 
\multirow{1}{*}{$L_{\textrm{agg}}$}                                     & 0.194          & 0.156          & 0.152          & 0.371          & 0.162          & 0.373          & 0.452           & \textbf{0.277}\\  
\multirow{1}{*}{Coreg}                                                  & 0.215          & 0.196          & 0.164          & 0.784          & 0.248          & 0.395          & 0.659           & 0.444\\            
\multirow{1}{*}{Heuristic}                                              & \textbf{0.192} & 0.218          & 0.198          & 0.697          & 0.280          & 0.474          & 0.515           & 0.400\\               
\multirow{1}{*}{TLMV     }                                              & 0.284          & 0.259          & 0.317          & 0.412          & 0.154          & 0.363          & 0.533           & 0.430\\      
\multirow{1}{*}{RMSC     }                                              & 0.254          & 0.255          & 0.194          & 0.407          & 0.173          & 0.422          & 0.507           & 0.279\\ 
\multirow{1}{*}{MT     }                                                & 0.249          & \textbf{0.133} & 0.158          & 0.544          & 0.103          & 0.371          & 0.436           & 0.298\\                       
\multirow{1}{*}{$L_{1}$ }                                               & 0.194          & 0.154          & 0.148          & 0.373          & 0.163          & 0.285          & \textbf{0.367}  & 0.440\\      
\multirow{1}{*}{$\boldsymbol{L_{-10}}$ {\scriptsize  (\textbf{ours}) }} & 0.200          & 0.159          & \textbf{0.144} & \textbf{0.368} & \textbf{0.095} & \textbf{0.283} & 0.374           & 0.439\\
\specialrule{1.5pt}{.1pt}{.1pt} 
\end{tabular}
\caption{Average Clustering Error}
\label{table:ExperimentsRealNetworksUnsigned}
\end{table*}
\section{Computing the smallest eigenvalues and eigenvectors of $M_p(A_1,\dots,A_T)$}\label{sec:PowerMethod}
We present an efficient
method for the computation of the smallest eigenvalues of $M_p(A_1,\dots,A_T)$
which does not require the computation of  the matrix $M_p(A_1,\dots,A_T)$. This is particularly important when dealing with large-scale problems as $M_p(A_1,\dots,A_T)$ is typically dense even though each $A_i$ is a sparse matrix. We restrict our attention to the case $p<0$ which is the most interesting one in practice. The positive case $p>0$ as well as the limit case $p\to 0$ deserve a different analysis and \pedro{are} not considered here. 

Let  $A_1,\dots,A_T$ be positive definite matrices. 
If ${\lambda_1\leq \cdots\leq \lambda_n}$ are the eigenvalues of $M_p(A_1,\dots,A_T)$  corresponding to the eigenvectors 
$\u_1, \dots, \u_n$, then $\mu_i\!=\!(\lambda_i)^p$, $i\!=\!1,\dots, n$, are the  eigenvalues of $M_p(A_1,\dots,A_T)^p$ corresponding to the eigenvectors $\u_i$. 
However, the function $f(x)\!=\!x^p$ 
is order reversing for $p<0$.
Thus, the relative ordering of the $\mu_i$'s changes into $\mu_1\geq \cdots\geq \mu_n$. Thus, the smallest eigenvalues and eigenvectors of $M_p(A_1,\dots,A_T)$ can be computed by addressing the largest ones of 
$M_p(A_1,\dots,A_T)^p$.  To this end we propose a power method type outer-scheme, combined with a Krylov subspace approximation inner-method. The pseudo code is presented in  Algs. \ref{alg:PM} and \ref{alg:kyrlov}.  
Each step of the outer iteration in Alg.\ \ref{alg:PM}   requires to compute  the $p$th power of $T$ matrices times a vector. Computing $A^p \times vector$, reduces to the problem of computing the product of a matrix function times a vector. Krylov methods are among the most efficient and most studied strategies to address such a computational issue. As $A^p$ is a polynomial in $A$, we apply a Polynomial Krylov Subspace Method (PKSM), whose pseudo code is presented in Alg. \ref{alg:kyrlov} and which we briefly describe in the following. For further details we refer to \cite{Higham:2008:FM} and the references therein.  For the sake of generality, below we describe the method for a general positive definite matrix $A$.

The general idea of PKSM $s$-th iteration is to project $A$ onto  the subspace 
$\mathbb K^s(A,\y) = \mathrm{span}\{\y, A\y, \dots, A^{s-1}\y\}$
and solve the problem there. The projection onto $\mathbb K^s(A,\y)$ is realized by means of the Lanczos process, 
producing a sequence of matrices $V_s$ with orthogonal columns,  
where the first column of $V_s$ is $\y/\vectornorm{\y}_2$ and $\mathrm{range}(V_s)=\mathbb K^s(A,\y)$. 
Moreover at each step we have $AV_s = V_s H_s + \v_{s+1}\e_{s}^T$
where $H_s$ is $s\times s$ symmetric tridiagonal, 
and $\e_i$ is the $i$-th canonical vector. 
The matrix vector product $\x = A^p\y$ is then approximated by 
$\x_s= V_s (H_s)^{p} \e_1\|\y\|\approx A^{p} \y$.

Clearly, if operations are done with infinite precision, the exact $\x$ is obtained after $n$ steps. However, in practice, the error $\|\x_s-\x\|$ decreases very fast with $s$ and often very few steps are enough to reach a desirable tolerance.  Two relevant observations are in order: 
first, the matrix  $H_{s}=V_{s}^T AV_{s}$  can be computed iteratively alongside the Lanczos method, thus it does not require any additional matrix multiplication; second, the $p$ power of the matrix $H_s$ can be computed directly without any notable increment in the algorithm cost,
since $H_{s}$ is tridiagonal of size $s \times s$.  

Several eigenvectors can be simultaneously computed with Algs.~\ref{alg:PM} and \ref{alg:kyrlov}
by orthonormalizing \pedro{the current eigenvector approximation} at every step of the power method (Alg.~\ref{alg:PM}) (see f.i. algorithm 5.1 Subspace iteration in~\cite{saad:2011:numerical}). 
Moreover, the outer iteration in Alg.\ \ref{alg:PM} can be easily run in parallel as the vectors $\mathbf u_k^{(i)}$, $i=1,\dots,T$ can be built independently of each other.

A numerical evaluation of Algs.~\ref{alg:PM} and \ref{alg:kyrlov} is presented in Fig.~\ref{fig:timeComparison}. 
We consider 
graphs of sizes 
${\left| V \right|\in\{1\!\times\!10^4,2\!\times\!10^4,3\!\times\!10^4,4\!\times\!10^4\}}$.
Further, for each multilayer graph we generate two assortative graphs with parameters $p_{\mathrm{in}}\!=\!0.05$ and $p_{\mathrm{in}}\!=\!0.025$, following the SBM.
Moreover, we consider the power mean Laplacian $L_p=M_p(L_\sym^{(1)},L_\sym^{(2)})$ with parameter ${p\!\in\!\{-1,-2,-5,-10\}}$.
As a baseline we take the arithmetic mean Laplacian ${L_1=M_1(L_\sym^{(1)},L_\sym^{(2)})}$ and use Matlab's \texttt{eigs} function.
For all cases, we compute the two eigenvectors corresponding to the smallest eigenvalues.
We present the mean execution time of 10 runs. Experiments are performed using one thread.

\section{Experiments}\label{sec:experiment}

%
We take 
the following baseline approaches of
spectral clustering applied to: the average adjacency matrix~($\mathbf{L_{agg}}$), the arithmetic mean Laplacian~($\mathbf{L_{1}}$),
the layer with the largest spectral gap~(\textbf{Heuristic}), and to the layer with the smallest clustering error~(\textbf{BestView}).
Further, we consider: 
Pairwise Co-Regularized Spectral Clustering~\cite{NIPS2011_4360}, with parameter $\lambda=0.01$ (\textbf{Coreg}), 
which proposes a spectral embedding 
generating
a clustering 
consistent among all graph layers,
Robust Multi-View Spectral Clustering~\cite{xia2014robust}, with parameter $\lambda = 0.005$ (\textbf{RMSC}),
which obtains a robust consensus representation by fusing noiseless information present among layers,
spectral clustering applied to a suitable convex combination of normalized adjacency matrices~\cite{zhou2007spectral} (\textbf{TLMV}),
and a tensor factorization method~\cite{Bacco:2017} (\textbf{MT}), 
which considers a multi-layer mixed membership (SBM).

We take several datasets:
\textit{3sources}\cite{liu2013multi}, \textit{BBC}\cite{Greene2005} and \textit{BBC Sports}\cite{greene2009matrix} news articles,
a dataset of Wikipedia articles\cite{rasiwasia2010new}, 
the hand written \textit{UCI} digits dataset with six different 
features 
and citations datasets
\textit{CiteSeer}\cite{lu:icml03},
\textit{Cora}\cite{mccallum2000automating} and
\textit{WebKB}\cite{Craven:1998:LES:295240.295725}, 
(from WebKB we only take the subset Texas).
%
%
For each layer we build the corresponding adjacency matrix from the $k$-nearest neighbour graph based
on the Pearson linear correlation between nodes, i.e. the higher the correlation the nearer the nodes are.
We test all clustering methods over all choices of ${k\in\{20,40,60,80,100\}}$,
and present the average clustering error in Table~\ref{table:ExperimentsRealNetworksUnsigned}.
Datasets CiteSeer, Cora and WebKB have two layers: one is a fixed citation network, whereas the second one is the $k$-nearest neighbour graph built on \pedro{documents} features. 
We can see that in four out of eight datasets the power mean Laplacian $L_{-10}$ gets the smallest clustering error.
The largest difference in clustering error is present in the UCI dataset,
where the second best is MT. Further, $L_1$ presents the smallest clustering error in Cora, being $L_{-10}$ close to it.
The smallest clustering error in WebKB is achieved by~$L_{\textrm{agg}}$.
This dataset 
is particularly challenging, due to conflictive \pedro{layers}\cite{He:2017}.

{\bf Acknowledgments}.
The work of P.M., A.G. and M.H. has been funded by the ERC starting grant NOLEPRO n. 307793. The work of F.T. has been funded by the Marie Curie Individual Fellowship MAGNET n. 744014.
\bibliography{references}
\bibliographystyle{abbrv}

\appendix

\section{Proofs for the Stochastic Block Model analysis}
This section has \pedro{two} parts corresponding to the Case 1 and 2 of the stochastic block model analysis. At the beginning of each of these sections, we first state what will be proved and discuss 
further refinements implied by the results presented here. For convenience we recall the notation where needed.

The correspondence between the results of the main paper and those proved here is as follows:
In Section \ref{SBM1} we discuss and prove Lemma 1, 2, Theorem 1 and Corollary 1 of the main paper. These results are directly implied by Lemma \ref{lemma:eigenvalues_and_eigenvectors_of_generalized_mean_V2}, Theorem \ref{SBMmain1} and Corollaries \ref{cor1}, \ref{cor2}, respectively, of the present manuscript. Then, in Section \ref{SBM2}, we prove Theorem 2 and Theorem 3 of the main paper which are respectively equivalent to Theorems \ref{thm:spectrum_L} and \ref{SBMmain2} below.

Before proceeding to the proofs, let us recall the setting. Let $V=\{v_1,\ldots,v_{n}\}$ be a set of nodes and let $T$ be the  number of layers,
represented by the adjacency  matrices $\multiLayerGraph{W} = \{ W^{(1)},\ldots,W^{(T)} \}$.
For each matrix $W^{(t)}$ we have a graph $G^{(t)} = (V,W^{(t)})$ and, overall, a multilayer graph ${\multiLayerGraph{G}=(G^{(1)}, \ldots,G^{(T)})}$.
We denote the ground truth clusters by $\mathcal{C}_1, \ldots, \mathcal{C}_k$ and assume that they all have the same size, 
i.e. $\abs{\mathcal{C}_i}=\abs{\mathcal{C}}$ for $i=1,\ldots,k$. 

In the following, we denote the identity matrix in $\R^m$ by $I_{m}$. Furthermore, for a matrix $X\in\R^{m\times m}$, we denote its eigenvalues by $\lambda_1(X),\ldots,\lambda_m(X)$.
\subsection{All layers have the same clustering structure}\label{SBM1}
For $t=1,\ldots,T$, let $\pp^{(t)}$ (resp. $\qp^{(t)}$) denote the probability that there exists an edge in layer $G^{(t)}$ between nodes that belong to the same (resp. different) clusters. Suppose that for $t=1,\ldots,T$, the expected adjacency matrix $\Wt\in\R^{n\times n}$ of $G^{(t)}$ is given for $i,j=1,\ldots,n$ as
$$\Wt_{ij} = \begin{cases}\pp^{(t)}&\text{if }v_i,v_j \text{ belong to the same cluster}\\ \qp^{(t)}&\text{otherwise}.\end{cases} $$
Furthermore, for every $t=1,\ldots,T,$ and $\epsilon\geq 0$, let
$$\mathcal{D}^{(t)}=\diag(\Wt\ones), \quad \rho_t= \frac{\pp^{(t)}-\qp^{(t)}}{\pp^{(t)}+(k-1)\qp^{(t)}},$$
$$ \lsym{t}=I_{n}-(\mathcal{D}^{(t)})^{-1/2}\Wt(\mathcal{D}^{(t)})^{-1/2}+\epsilon I_{n}$$

\pedro{
Observe that $\lsym{t}$ is the normalized Laplacian of the expected graph plus a diagonal shift. The diagonal shift is necessary to enforce this matrix to be positive definite for the cases $p\leq 0$, as stated in~\cite{bhagwat_subramanian_1978}. 
}

We consider the vectors $\bchi_1,\ldots, \bchi_k\in\R^{n}$ defined as 
\begin{align*}
   \boldsymbol \chi_1  = \one,\quad 
   \boldsymbol \chi_i = (k-1)\one_{\mathcal{C}_i}-\one_{\overline{\mathcal{ C}_i}},\quad i = 2,\ldots,k.
\end{align*}
By construction, $\bchi_1,\ldots, \bchi_k$ are all eigenvectors of $\Wt$ for every $t=1,\ldots,T$. These eigenvectors are precisely the vectors allowing to recover the ground truth clusters. 
Let
$$\L{p}=M_p\big(\lsym{1},\ldots,\lsym{T}\big)$$
where we assume that $\epsilon>0$ if $p\leq 0$. We prove the following:

\begin{thm}\label{SBMmain1}
Let $p\in[-\infty,\infty]$, and assume that $\epsilon>0$ if $p\leq 0$. Then, there exists $\lambda_i$ such that $\L{p}\bchi_i=\lambda_i\bchi_i$ for all $i=1,\ldots,k$. Furthermore, $\lambda_1,\ldots,\lambda_k$ are the $k$-smallest eigenvalues of $\L{p}$ if and only if $m_p(\boldsymbol{\mu}+\epsilon\ones)<1+\epsilon$, where $\boldsymbol{\mu}=(1-\rho_1,\ldots,1-\rho_T)$.
\end{thm}

Before giving a proof of Theorem \ref{SBMmain1} we discuss some of its implications in order to motivate the result. First, we note that it implies that if $\bchi_1,\ldots, \bchi_k$ are among the smallest eigenvectors of $\L{p}$ then they  are among the smallest eigenvectors of $\L{q}$ for any $q\leq p$.
%
\begin{corollary}\label{cor1}
Let $q\leq p$ and assume that $\epsilon>0$ if $\min\{p,q\}\leq 0$.
If $\bchi_1,\ldots,\bchi_k$ correspond to the $k$-smallest eigenvalues of $\L{p}$, then
$\bchi_1,\ldots,\bchi_k$ correspond to the $k$-smallest eigenvalues of $\L{q}$.
\end{corollary}
\begin{proof}
If $\lambda_1,\ldots,\lambda_k$ are among the $k$-smallest eigenvalues of $\L{p}$, then by Theorem \ref{SBMmain1}, we have $m_p(\boldsymbol{\mu}+\epsilon\ones)<1+\epsilon$. As $m_q(\boldsymbol{\mu}+\epsilon\ones)\leq m_p(\boldsymbol{\mu}+\epsilon\ones)$, Theorem \ref{SBMmain1} concludes the proof.
\end{proof}

The next corollary deals with the extreme cases where $p=\pm\infty$. In particular, it implies that whenever at least one layer $G^{(t)}$ is informative then the eigenvectors of $\L{-\infty}$ allow to recover the clusters. This contrasts with $p=\infty$ where the clusters can be recovered from the eigenvectors of $\L{\infty}$ if and only if all layers are informative. 

\begin{corollary}\label{cor2}
Let $p\in[-\infty,\infty]$ and $\epsilon>0$ if $p\leq 0$.
\begin{enumerate}[topsep=-3pt,leftmargin=*,resume]\setlength\itemsep{-3pt}
\item If $p=\infty$, then $\bchi_1,\ldots,\bchi_k$ correspond to the $k$-smallest eigenvalues of $\L{\infty}$ if and only if 
all layers are informative, i.e.
$\pp^{(t)} > \qp^{(t)}$ holds for all $t\in\{1,\ldots,T\}$.
\item If $p=-\infty$, then $\bchi_1,\ldots,\bchi_k$ correspond to the $k$-smallest eigenvalues of $\L{-\infty}$ if and only if 
there is at least one informative layer, i.e.
 there exists a $t\in\{1,\ldots,T\}$ such that $\pp^{(t)} > \qp^{(t)}$. 
\end{enumerate}
\end{corollary}

\begin{proof}
Recall that $\lim_{p\to\infty}m_p(\v)=\max_{i=1,\ldots,m} v_i$ and $\lim_{p\to-\infty}m_p(\v)=\min_{i=1,\ldots,m} v_i$ for any $\v\in\R^m$ with nonnegative entries. 
Hence, we have $m_{\pm\infty}(\boldsymbol{\mu}+\epsilon\ones)=m_{\pm\infty}(\boldsymbol{\mu})+\epsilon$ and thus the condition $m_p(\boldsymbol{\mu}+\epsilon\ones)<1+\epsilon$ of Theorem \ref{SBMmain1} reduces to $m_{\pm\infty}(\boldsymbol{\mu})<1$ for $p=\pm\infty$. 
Furthermore, note that we have $\mu_t=1-\rho_t<1$ if and only if $\pp^{(t)}>\qp^{(t)}$. 
To conclude, note that $m_{\infty}(\boldsymbol{\mu})=\max_{t=1,\ldots,T}\mu_t<1$ if and only if $\mu_t<1$ for all $t=1,\ldots,T$ and $m_{-\infty}(\boldsymbol{\mu})=\min_{t=1,\ldots,T}\mu_t<1$ if and only if there exists $t\in\{1,\ldots,T\}$ such that $\mu_t<1$.
\end{proof}
For the proof of Theorem \ref{SBMmain1}, we give an explicit formula for eigenvalues of $\L{p}$ in terms of the eigenvalues of $\lsym{1},\ldots,\lsym{T}$. 
Then, we discuss the ordering of these eigenvalues. Furthermore, we show that $\bchi_i$ are all eigenvectors of $\L{p}$ and compute their corresponding eigenvalues.

By construction, there are $k$ eigenvectors $\boldsymbol\chi_i$ of $\Wt$ corresponding to a possibly nonzero eigenvalue $\lambda_i^{(t)}$. These are given by
\begin{align*}
&   \boldsymbol \chi_1  = \one, \qquad \lambda_1^{(t)} = \abs{\mathcal{C}}(\pp^{(t)}+(k-1)\qp^{(t)}),\\
 &  \boldsymbol \chi_i = (k-1)\one_{\mathcal{C}_i}-\one_{\overline{\mathcal{ C}_i}},
     \quad \lambda_i^{(t)}= \abs{\mathcal{C}}(\pp^{(t)}-\qp^{(t)})\\
\end{align*}
for $i=2,\ldots,k$. It follows that $\boldsymbol\chi_1,\ldots,\boldsymbol\chi_k$ are eigenvectors of $\lsym{t}$ with eigenvalues $\lambda_i(\lsym{t}) $. Furthermore, we have
\begin{align}\label{eigsLQsym}
    &\lambda_1(\lsym{t}) = \pedro{\epsilon},  \quad \lambda_i(\lsym{t}) = 1 -\rho_t\pedro{+\epsilon}, \quad i = 2,\ldots,k,\notag\\
    &\lambda_j(\lsym{t})=1\pedro{+\epsilon}  , \quad j=k+1,\ldots,n
\end{align}
Let
$$\L{p}=M_p\big(\lsym{1},\ldots,\lsym{T}\big).$$
The following lemma will be helpful to show that $\boldsymbol \chi_1,\ldots,\boldsymbol \chi_k$ are all eigenvectors $\L{p}$ and gives a formula for their corresponding eigenvalue.
\newcommand{\psd}{symmetric positive semi-definite}
\begin{lemma}\label{lemma:eigenvalues_and_eigenvectors_of_generalized_mean_V2}
 Let $A_1 , \ldots , A_T\in\R^{n\times n}$ be \psd{} matrices and let $p\in\R$. Suppose that $A_1,\ldots,A_T$ are positive definite if $p\leq 0$. If $\u$ is an eigenvector of $A_i$ with corresponding eigenvalue $\lambda_i$ for all $i=1,\ldots,T$,
then $\u$ is an eigenvector of $M_p(A_1 , \ldots , A_T)$ with eigenvalue $m_p(\lambda_1 , \ldots , \lambda_T)$.
\end{lemma}
\begin{proof}
First, note that $M=M_p(A_1 , \ldots , A_T)$ is symmetric positive (semi-)definite as it is a positive sum of such matrices. In particular, $M$ is diagonalizable and thus the eigenvectors of $M$ and $M^p$ are the same for every $p$. Now, as
 $A_i \u = \lambda_i \u$ for $i=1,\ldots,T$,
we have $A_i^p \u = \lambda_i^p \u$ for all $i$ and thus
 \begin{align*}
   M_p^p(A_1 , \ldots , A_T)\u &= {\frac  {1}{T}} \sum_{{i=1}}^{T} A_i^p \u = {\frac  {1}{T}} \sum_{{i=1}}^{T} \lambda_i^p \u \\&= m_p^p(\lambda_1 , \ldots , \lambda_T)\u.
 \end{align*}
Thus, $\u$ is an eigenvector of $M_p(A_1 , \ldots , A_T)$ with eigenvalue $m_p(\lambda_1 , \ldots , \lambda_T)$. 
\end{proof}
The above lemma, allows to obtain an explicit formula for $\L{p}$ which fully describes its spectrum. Indeed, we have the following
\begin{corollary}\label{specdec}
Let $X$ and $\Lambda^{(1)}$ be matrices such that $\lsym{1}=X\Lambda^{(1)}X^T$, $X$ is orthogonal and $\Lambda^{(1)}=\diag( \lambda_1(\lsym{1}),\ldots, \lambda_n(\lsym{1}))$.
 Then, we have
$\L{p}=X\Lambda X$ where $\Lambda$ is the diagonal matrix $\Lambda=\diag(\lambda_1(\L{p}),\ldots,\lambda_n(\L{p}))$ 
with 
$\lambda_i(\L{p})=m_p(\lambda_i(\lsym{1}),\ldots,\lambda_i(\lsym{T}))$,
for all $i=1,\ldots,n$.
\end{corollary}
\begin{proof}
As $\lsym{t}$ have the same eigenvectors for every $t=1,\ldots,T$, it follows by Lemma \ref{lemma:eigenvalues_and_eigenvectors_of_generalized_mean_V2} that $\L{p}X=X\Lambda$ and thus $\L{p}=X\Lambda X^\top$.
\end{proof}
We note that on top of providing information on the spectral properties of $\L{p}$, Corollary \ref{specdec} ensures the existence of $\L{\pm\infty}\in\R^{n\times n}$ such that $\lim_{p\to\pm\infty}\L{p}=\L{\pm\infty}$. 

Combining Lemma \ref{lemma:eigenvalues_and_eigenvectors_of_generalized_mean_V2} with equation \eqref{eigsLQsym} we obtain the following 
\begin{lemma}\label{lemain1}
The limits $\lim_{p\to \pm\infty} \L{p}=\L{\pm\infty}$ exist. Furthermore, for $p\in[-\infty,\infty]$, we have $\L{p}\bchi_i=\lambda_i\bchi_i$ with 
$$ \lambda_1 = \epsilon, \quad \lambda_i = m_p(\boldsymbol{\mu}+\epsilon\ones), \quad i = 2,\ldots,k $$
%
where $\boldsymbol{\mu}=(1-\rho_1,\ldots,1-\rho_T)$. Furthermore, the remaining eigenvalues satisfy $\lambda_{k+1}=\cdots=\lambda_{n}=1+\epsilon$.
\end{lemma}
\begin{proof}
With Corollary \ref{specdec} and Equation \eqref{eigsLQsym} we directly obtain for $i=2,\ldots,k$ and $j=k+1,\ldots,n$,
\begin{align*}
\lambda_1 &= \Big( \frac{1}{T} \sum_{t=1}^T \epsilon^p \Big)^{1/p} = m_p(\epsilon\ones)=\epsilon \\
\lambda_i &= \Big( \frac{1}{T}\sum_{t=1}^T (1-\rho_t+\epsilon)^p \Big)^{1/p} = m_p(\boldsymbol{\mu}+\epsilon\ones)\\
\lambda_j &= \Big( \frac{1}{T} \sum_{t=1}^T (1+\epsilon)^p  \Big)^{1/p} = m_p((1+\epsilon)\ones)=1+\epsilon \\
\end{align*}
\end{proof}
We are now ready to prove Theorem \ref{SBMmain1}.
%
%
\begin{proof}[Proof of Theorem \ref{SBMmain1}]
Clearly, $\lambda_1,\ldots,\lambda_k$ are among the $k$-smallest eigenvalues of $\L{p}$ if and only if 
$\lambda_1<\lambda_{k+1}\leq \cdots \leq \lambda_{n}$ 
and
$\lambda_2\leq\ldots \leq \lambda_k<\lambda_{k+1}\leq \cdots \lambda_{n}$ where $\lambda_1,\ldots,\lambda_n$ are all eigenvalues of $\L{p}$. 
%
By Lemma~\ref{lemain1}, we have $\lambda_1=\epsilon$, $\lambda_2=\cdots= \lambda_k =  m_p(\boldsymbol{\mu}+\epsilon\ones)$ and $\lambda_{k+1}=\cdots=\lambda_{n}=1+\epsilon$. 
Clearly $\lambda_1 =\epsilon < 1+\epsilon = \lambda_{k} \leq \cdots \leq \lambda_{n}$, thus, the first condition holds.
Hence, $\lambda_1,\ldots, \lambda_k$ correspond to the $k$-smallest eigenvalues of $\L{p}$ if and only if $m_p(\boldsymbol{\mu}+\epsilon\ones)<1+\epsilon$ which concludes the proof.
\end{proof}

\subsection{No layer contains full information of the Graph}\label{SBM2}
In this setting, we fix the number $k$ of cluster to $k=3$. 

For convenience, we slightly overload the notation for the remaining of this section: we denote by $n$ the size of each cluster $\mathcal C_1, \dots, \mathcal C_k$, i.e. $\abs{\mathcal C_i}=\abs{\mathcal{C}}=n$ for $i=1,\dots, k$. 
Thus, the size of the graph is expressed in terms of the number and size of clusters, i.e. $\abs{V}=nk$.

Furthermore, we suppose that for $t=1,2,3$, the expected adjacency matrix $\Wt\in\R^{3n\times 3n}$ of $G^{(t)}$, are given, for all $i,j=1,\ldots,3n$, as
\begin{equation*}
\Wt_{ij} = \begin{cases}\pin&\text{if }v_i,v_j \in \mathcal{C}_t \text{ or } v_i,v_j\in \overline{\mathcal{C}_t}\\
\pout&\text{otherwise},\end{cases}
\end{equation*}
where $0<\pout\leq \pin\leq 1$. For $t=1,2,3$ and $\epsilon\geq 0$, let
$\mathcal{D}^{(t)}=\diag(\Wt\ones)$,
$$\lsym{t}=I-(\mathcal{D}^{(t)})^{-1/2}\Wt(\mathcal{D}^{(t)})^{-1/2}+\epsilon I,$$
and for a nonzero integer $p$ let
$$\L{p}=M_p(\lsym{1},\lsym{2},\lsym{3}),$$
where we assume that $\epsilon>0$ if $p<0$. Consider further $\bchi_1,\bchi_2,\bchi_3\in\R^{3n}$ the vectors defined as
\begin{equation*}
\bchi_1 =\ones, \quad \bchi_2= \ones_{\mathcal{C}_1}-\ones_{\mathcal{C}_2},\quad \bchi_3= \ones_{\mathcal{C}_1}-\ones_{\mathcal{C}_3}.
\end{equation*}
In opposition to the previous model, it turns out that $\lsym{1},\lsym{2}$ and $\lsym{3}$ do not commute and thus do not share the same eigenvectors. Hence, we can not derive an explicit expression for $\L{p}$ as in Corollary \ref{specdec}. In particular this implies that we need to use different mathematical tools in order to study the eigenpairs of $\L{p}$. 

The first main result of this section, presented in Theorem \ref{thm:spectrum_L}, shows that, in general, the ground truth clusters can not be reconstructed from the $3$ smallest eigenvectors of $\lsym{t}$ for any $t=1,2,3$.
\begin{thm}\label{thm:spectrum_L}
If $1\geq\pp>\qp>0$, then for any $t=1,2,3$, there exist scalars $\alpha>0$ and $\beta>0$ such that the eigenvectors of $\mathcal L_\sym^{(t)}$ corresponding to the two smallest eigenvalues are
$$
\boldsymbol \varkappa_1 = \alpha \one_{\mathcal C_t}+\one_{\overline{\mathcal C_t}} \quad \text{and} \quad
\boldsymbol \varkappa_2 =-\beta \one_{\mathcal C_t}+\one_{\overline{\mathcal C_t}}
$$
whereas any vector orthogonal to both $\boldsymbol \varkappa_1$ and $\boldsymbol \varkappa_2$ is an eigenvector for the third smallest eigenvalue. 
\end{thm}
In fact, we prove even more by giving a full description of the eigenvectors of $\lsym{t}$ as well as the ordering of their corresponding eigenvalues. 
These results can be found in Lemma \ref{LsymWn} below. 

Our second main result is the following Theorem \ref{SBMmain2}. It shows that the ground truth clusters can always be recovered from the three smallest eigenvectors of $\L{p}$.
\begin{thm}\label{SBMmain2}
Let $p$ be any nonzero integer and assume that $\epsilon>0$ if $p< 0$. Furthermore, suppose that $0<\pout<\pin\leq 1$. Then, there exists $\lambda_i$ such that $\L{p}\bchi_i=\lambda_i\bchi_i$ for $i=1,2,3$ and $\lambda_1,\lambda_2,\lambda_3$ are the three smallest eigenvalues of $\L{p}$. 
\end{thm}
Again, we actually prove more than just Theorem \ref{SBMmain2}. In fact, a full description of the eigenvectors of $\L{p}$ and of the ordering of their corresponding eigenvalues is given in Lemma \ref{specmainthm} below.

For the proof of Theorems \ref{thm:spectrum_L} and \ref{SBMmain2}, and the corresponding additional results, we proceed as follows. First we assume that $n=|\mathcal{C}_i|=1$ and prove our claims. Then, we generalize these results to the case $n>1$. For the sake of clarity, as we will need to refer to the case $n=1$ for the proofs of the case $n> 1$, we put a tilde on the matrices in $\R^{3\times 3}$.

\underline{\textit{The case $n=1$:}}

Suppose that $n=1$, then $\tilde{\mathcal{L}}_{\sym}=\lsym{1}$ is given by
\begin{equation*}
\tilde{\mathcal{L}}_{\sym}=\tau I_3 -\tilde{\mathcal{D}}^{-1/2}\tilde{\mathcal{W}}\tilde{\mathcal{D}}^{-1/2} =\tau I_3 -\tilde{\mathcal{M}},
\end{equation*} 
where $\tau=1+\epsilon$, $\tilde{\mathcal{W}}=\mathcal{W}^{(1)}$, 
$\tilde{\mathcal{D}}=\diag(\tilde{\mathcal{W}}\ones)$,
\begin{equation}\label{defM}
\tilde{\mathcal{D}}= \begin{pmatrix}\alpha & 0 & 0 \\0 & \beta &0 \\ 0 & 0 & \beta\end{pmatrix}, \qquad \tilde{\mathcal{M}}=\begin{pmatrix}a & b & b \\b & c &c \\ b & c & c\end{pmatrix},
\end{equation}
and $\alpha,\beta, a,b,c> 0$ are given by
\begin{align*}
& \alpha = \pin+2\pout, \qquad \beta = 2\pin+\pout,\\
& a = \frac{\pin}{\alpha}, \qquad b=\frac{\pout}{\sqrt{\alpha\beta}}, \qquad c = \frac{\pin}{\beta}.
\end{align*}

Moreover, note that for any $(\lambda,\v)\in\R\times \R^3$ we have
\begin{equation}\label{connectMW}
\tilde{\mathcal{M}}\v = \lambda \v \qquad \iff \qquad \tilde{\mathcal{L}}_{\sym}\v = (\tau-\lambda)\v.
\end{equation}
This implies that we can study the spectrum of $\tilde{\mathcal{M}}$ in order to obtain the spectrum of $\tilde{\mathcal{L}}_{\sym}$. We have the following lemma:
\begin{lemma}\label{specM}
Suppose that $\pout>0$ and let $\Delta>0$ be defined as $ \Delta = \sqrt{(a-2c)^2 +8 b^2}$.
Then the eigenvalues of $\tilde{\mathcal{M}}$ are 
\begin{align*}
\tilde\lambda_1=0,\qquad \tilde\lambda_2 = \frac{a+2 c-\Delta}{2}, \qquad \tilde \lambda_3 = 1, 
\end{align*}
and it holds $\tilde\lambda_1<\tilde\lambda_2<\tilde\lambda_3$. Furthermore, the corresponding eigenvectors are given by
\begin{align*}
&\u_1= (0,-1,1)^\top,\qquad \u_2 = \Big(\frac{a-2c-\Delta}{2b},1,1\Big)^\top, \\&\u_3 = \big(\sqrt\alpha,\sqrt\beta,\sqrt\beta\big)^\top, 
\end{align*}
and it holds $\frac{a-2c-\Delta}{2b}<0$.
\end{lemma}
\begin{proof}
The equality $\tilde{\mathcal{M}}\u_1=0$ follows from a direct computation. Furthermore, note that $\u_3=\tilde{\mathcal{D}}^{1/2}\ones$ and so
\begin{align*}
\tilde{\mathcal{M}}\u_3&=\tilde{\mathcal{D}}^{-1/2}\tilde{\mathcal{W}}\tilde{\mathcal{D}}^{-1/2}\mathcal D^{1/2}\ones 
= \tilde{\mathcal{D}}^{-1/2}\tilde{\mathcal{W}}\ones =\u_3
\end{align*}
implying $\tilde{\mathcal{M}}\u_3 =\u_3$. Now, let $s_{\pm}=\frac{a-2c\pm\Delta}{2b}$. Then $s_+$ and $s_-$ are the solutions of the quadratic equation $bs^2+(2c-a)s-2b=0$ which can be rearranged as $as+2b=(bs+2c)s.$ The latter equation is equivalent to
\begin{equation*}
\begin{cases}as+2b = \lambda s\\ bs+2c = \lambda\end{cases}\quad \iff\quad \tilde{\mathcal{M}}\begin{pmatrix}s\\ 1\\1\end{pmatrix}= \lambda\begin{pmatrix}s\\ 1\\1\end{pmatrix}.
\end{equation*}
Hence, $\u_{\pm}=(s_{\pm},1,1)$ are both eigenvectors of $\tilde{\mathcal{M}}$ corresponding to the eigenvalues 
\begin{equation*}
\lambda_{\pm} = b\, s_{\pm}+2c= \frac{a-2c\pm\Delta}{2}+2c=\frac{a+2 c\pm\Delta}{2}.
\end{equation*}
Note in particular that we have $\u_2=\u_-$ and $\tilde\lambda_2=\lambda_-$. This concludes the proof that $(\lambda_i,\u_i)$ are eigenpairs of $\tilde{\mathcal{M}}$ for $i=1,2,3$. We now show that $\tilde\lambda_1<\tilde\lambda_2<\tilde\lambda_3$ and $(a-2c-\Delta)/2b<0$.

As $\Delta>0$, we have $\lambda_-<\lambda_+$. We prove $\lambda_->0$. As $\pin>\pout$ by assumption, the definition of $a,b,c>0$ implies that
\begin{align*}
b^2 &= \frac{\pout^2}{(2 \pin+\pout)(\pin+2 \pout)}\\& < \frac{\pin^2}{(2 \pin+\pout)(\pin+2 \pout)} =ac.
\end{align*}
And from $ac>b^2$ it follows that
$a^2+4ac+4c^2>a^2-4ac+4c^2+8b^2 $ which implies that $(a+2c)^2>(a-2c)^2+8b^2=\Delta^2.$
Hence, $a+2 c-\Delta>0$ and thus $\lambda_->0$. Thus we have $0<\lambda_-<\lambda_+$. Now, as $\tilde{\mathcal{M}}$ has strictly positive entries, the Perron-Frobenius theorem (see for instance Theorem 1.1 in \cite{Francesco}) implies that 
$\tilde{\mathcal{M}}$ has a unique nonnegative eigenvector $\u$. Furthermore, 
$\u$ has positive entries and its corresponding eigenvalue is the spectral radius of $\tilde{\mathcal{M}}$. 
As $\u_3=\tilde{\mathcal{D}}^{1/2}\ones$ has positive entries and is an eigenvector of $\tilde{\mathcal{M}}$, we have $\u=\u_3$. It follows that 
$\rho(\tilde{\mathcal{M}})=\lambda_+=\tilde\lambda_3$. 
Furthermore, $\u_2$ must have a strictly negative entry and thus it holds $s_-<0$.
\end{proof}
Combining the results of Lemma \ref{specM} and \pedro{E}quation \eqref{connectMW} we directly obtain the following corollary which fully describes the eigenvectors of $\tilde{\mathcal{L}}_{\sym}$ as well as the ordering of the corresponding eigenvalues:
\begin{corollary}\label{LsymW3}
There exists $\tilde \lambda \in (0,1)$ and $s_-<0<s_+<1$ such that
$\big(\tau-1,(s_+,1,1)^\top\big),  \big(\tau-\tilde\lambda,(s_-,1,1)^\top\big), \big(\tau,(0,-1,1)^\top\big)
$ are the eigenpairs of $\tilde{\mathcal{L}}_{\sym}$.
\end{corollary}
\begin{proof}
The only thing which is not directly implied by Lemma \ref{specM} and \pedro{E}quation \eqref{connectMW} is that $s_+<1$. But this follows again from Lemma \ref{specM}. 
Indeed, as $(s_+,1,1)$ and $(\sqrt \alpha, \sqrt \beta, \sqrt \beta)$ must span the same line, we have 
$$s_+=\sqrt{\frac{\alpha}{\beta}}=\sqrt{\frac{\pin+2\pout}{2\pin+\pout}}.$$
As $\pout<\pin$, we get $0<s_+<1$.
\end{proof}

Now, we study the spectral properties of $\Ln{p}=\L{p}\in\R^{3\times 3}$. To this end, for $t=1,2,3$ let $\tilde{\mathcal{W}}^{(t)}={\mathcal{W}^{(t)}},\lsymn{t}=\lsym{t}\in\R^{3\times 3}$. Furthermore, consider the permutation matrices $\tilde P_1,\tilde P_2,\tilde P_3\in\R^{3\times 3}$ defined as
$$
\tilde P_1=I_3,\quad \tilde P_2 = \begin{pmatrix} 0 & 0 & 1 \\ 0 & 1 & 0 \\ 1 & 0 & 0\end{pmatrix},\quad \tilde P_3 = \begin{pmatrix} 0 & 1& 0 \\ 1 & 0 & 0 \\ 0 & 0 &1\end{pmatrix}.$$
Then, we have $\tilde{\mathcal{W}}^{(t)}=\tilde P_t\tilde{\mathcal{W}}\tilde P_t$ for $t=1,2,3$. The following lemma relates $\lsymn{t}$ and $\tilde{\mathcal{L}}_{\sym}$.
\begin{lemma}\label{permutelaplacian}
For $t=1,2,3$, we have $\tilde P_t=\tilde P_t^{-1}=\tilde P_t^\top$ and $\lsymn{t}=\tilde P_t\tilde{\mathcal{L}}_{\sym}\tilde P_t$.
\end{lemma}
\begin{proof}
The identity $\tilde P_t=\tilde P_t^{-1}=\tilde P_t^\top$ follows by a direct computation. 
Now, as $\tilde P_t\ones=\ones$, we have $\tilde P_t \tilde{\mathcal{W}}\tilde P_t\ones = \tilde P_t \tilde{\mathcal{W}}\ones$. Assuming the exponents on the vector in the following expressions are taken component wise, we have $\diag(\tilde{\mathcal{W}}\ones)^{-1/2}=\diag\big((\tilde{\mathcal{W}}\ones)^{-1/2}\big)$ and thus 
\begin{align*}
&\diag(\tilde P_t \tilde{\mathcal{W}}\tilde P_t\ones)^{-1/2}=\diag\big((\tilde P_t \tilde{\mathcal{W}}\tilde P_t\ones)^{-1/2}\big)\\&\quad=\diag\big(\tilde P_t(\tilde{\mathcal{W}}\ones)^{-1/2}\big)=\tilde P_t\diag\big((\tilde{\mathcal{W}}\ones)^{-1/2}\big)\tilde P_t\\&\quad=\tilde P_t\diag(\tilde{\mathcal{W}}\ones)^{-1/2}\tilde P_t=\tilde P_t\tilde{\mathcal{D}}^{-1/2}\tilde P_t.
\end{align*}
It follows that
\begin{align*}
\lsymn{t} &= \tau \tilde P_t\tilde P_t -\tilde P_t \tilde{\mathcal{D}}^{-1/2}\tilde P_t \tilde P_t \tilde{\mathcal W} \tilde P_t \tilde P_t \tilde{\mathcal{D}}^{-1/2}\tilde P_t \\
  &=\tilde P_t \tilde{\mathcal{L}}_{\sym}\tilde P_t,
\end{align*}
which concludes our proof.
\end{proof}
Combining Corollary \ref{LsymW3} with Lemma \ref{permutelaplacian}, we directly obtain the following
\begin{corollary}
There exists $\tilde \lambda \in (0,1)$ and $s_-<0<s_+$ such that
$\big(\tau-1,\tilde P_t(s_+,1,1)^\top\big), \big(\tau-\lambda,\tilde P_t(s_-,1,1)^\top\big),\big(\tau,\tilde P_t(0,-1,1)^\top\big)$
are the eigenpairs of $\lsymn{t}$ for $t=1,2,3$.
\end{corollary}
A similar argument as in the proof of Lemma \ref{lemma:eigenvalues_and_eigenvectors_of_generalized_mean_V2} implies that the eigenvectors of $\Ln{p}$ coincide with those of the matrix $\Lnp{p}\in\R^{3\times 3}$ defined as
\begin{equation*}
\Lnp{p}=(\lsymn{1})^p+(\lsymn{2})^p+(\lsymn{3})^p=3\Ln{p}^p.
\end{equation*}
We study the spectral properties of $\Lnp{p}$. To this end, we consider the following subspaces of matrices:
\begin{align*}
&\mathcal{U}_3 = \Big\{\begin{pmatrix} s_1 & s_2 & s_2 \\ s_3 & s_5 & s_4 \\ s_3 & s_4 & s_5\end{pmatrix}\ \Big|\ s_1,\ldots,s_5\in\R\Big\}, \\ 
&\mathcal{Z}_3 = \Big\{\begin{pmatrix} t_1 & t_2 & t_2 \\ t_2 & t_1 & t_2 \\ t_2 & t_2 & t_1\end{pmatrix}\ \Big|\ t_1,t_2\in\R\Big\}.
\end{align*}
We prove that for every $p$, it holds $(\lsymn{1})^p\in\mathcal{U}_3$ and $\Lnp{p}\in \mathcal{Z}_3$. We need the following lemma:
\begin{lemma}\label{lemU}
The following holds:
\begin{enumerate}[topsep=-3pt,leftmargin=*]
\item For all $\tilde A,\tilde B\in \mathcal U_3$ we have $\tilde A\tilde B\in \mathcal U_3$.\label{multiU3}
\item If $\tilde A\in\mathcal U_3$ and $\det(\tilde A)\neq 0$, then $\tilde A^{-1}\in \mathcal U_3$.\label{invU3}
\item $\mathcal{Z}_3=\tilde P_1\mathcal{U}_3\tilde P_1+\tilde P_2\mathcal{U}_3\tilde P_2+\tilde P_3\mathcal{U}_3\tilde P_3$.\label{U3toZ3}
\end{enumerate}
\end{lemma}
\begin{proof}
Let $\tilde A\in \mathcal{U}_3,\tilde C\in\mathcal{Z}_3$ be respectively defined as
\begin{equation*}
\tilde A=\begin{pmatrix} s_1 & s_2 & s_2 \\ s_3 & s_5 & s_4 \\ s_3 & s_4 & s_5\end{pmatrix}, \quad \tilde C=\begin{pmatrix} t_1 & t_2 & t_2 \\ t_2 & t_1 & t_2 \\ t_2 & t_2 & t_1\end{pmatrix}
\end{equation*}
\begin{enumerate}[topsep=-3pt,leftmargin=*]
\item Follows from a direct computation.
\item If $\det(\tilde A)\neq 0$, then $\tilde A$ is invertible and
\begin{align*}&\det(\tilde A)\tilde A^{-1}=\\&\begin{pmatrix}
 s_5^2-s_4^2 & s_2 (s_4-s_5) &s_2 (s_4-s_5) \\
 s_3(s_4-s_5) & s_1 s_5-s_2 s_3 & s_2 s_3-s_1 s_4 \\
  s_3(s_4-s_5) & s_2 s_3-s_1 s_4 & s_1 s_5-s_2 s_3
 \end{pmatrix}.
  \end{align*}
 It follows that $\tilde A^{-1}\in\mathcal{U}_3.$
 \item We have
 \begin{align}\label{symact}
& \sum_{i=1}^3\tilde P_i\tilde A\tilde P_i=\\ &\begin{pmatrix}
 s_1+2 s_5 & s_2+s_3+s_4 & s_2+s_3+s_4 \\
  s_2+s_3+s_4 & s_1+2 s_5 & s_2+s_3+s_4 \\
  s_2+s_3+s_4 & s_2+s_3+s_4 & s_1+2 s_5 \end{pmatrix}\notag
 \end{align}
and conversely, there clearly exists $s_1,\ldots,s_4$ such that $s_1+2s_5=t_1$ and $s_2+s_3+s_4=t_2$, so we have $ \sum_{i=1}^3\tilde P_i\tilde A\tilde P_i=\tilde C$ implying the reverse inclusion.
\end{enumerate}
\end{proof}
Now, we show that $\Lnp{p}\in\mathcal{Z}_3$ for all nonzero integer $p$.
\begin{lemma}\label{LpinZ3}
For every integer $p\neq 0$ we have $\Lnp{p}\in \mathcal{Z}_3$.
\end{lemma}
\begin{proof}
From \eqref{defM}, we know that $\lsymn{1}\in \mathcal{U}_3$. 
By point \ref{invU3} in Lemma \ref{lemU}, this implies that $(\lsymn{1})^{\sign(p)}\in \mathcal{U}_3$. 
Now point \ref{multiU3} of Lemma \ref{lemU} implies that $(\lsymn{1})^p=\big((\lsymn{1})^{\sign(p)}\big)^{|p|}\in \mathcal{U}_3$. 
Finally, by Lemma \ref{permutelaplacian} and point \ref{U3toZ3} in Lemma \ref{lemU}, we have $$\Lnp{p}=\sum_{t=1}^3(\lsymn{t})^p=\sum_{t=1}^3\tilde P_t(\lsymn{1})^p\tilde P_t\in\mathcal{Z}_3,$$
which concludes the proof.
\end{proof}
Matrices in $\mathcal{Z}_3$ have the interesting property that they have a simple spectrum and they all share the same eigenvectors. Indeed we have the following:
\begin{lemma}\label{specZ3}
Let $\tilde C\in\mathcal{Z}_3$ and $t_1,t_2$ be such that $\tilde C=(t_1-t_2)I_3+t_2\tilde E$ where $\tilde E\in\R^{3\times 3}$ is the matrix of all ones. Then the eigenpairs of $\tilde C$ are given by:
\begin{align*}
&\big(t_1-t_2,(-1,0,1)^\top\big),\quad \big(t_1-t_2,(-1,1,0)^\top\big),\\ & \big(t_1+2t_2,(1,1,1)^\top\big).
\end{align*}
\end{lemma}
\begin{proof}
Follows from a direct computation.
\end{proof}

So, the last thing we need to discuss is the order of the eigenvalues of $\Lnp{p}$. To this end, we study the sign pattern of the powers of this matrix.
\begin{lemma}\label{shitmatrix}
For every positive integer $p>0$ we have $(\tilde{\mathcal{L}}_{\sym}^p)_{i,j}<0<(\tilde{\mathcal{L}}_{\sym}^p)_{i,i}<\tau^p$ for all $i,j=1,2,3$ with $i\neq j$. For every negative integer $p<0$ we have $(\tilde{\mathcal{L}}_{\sym}^p)_{i,j}>0$ for all $i,j=1,2,3$.
\end{lemma}
\begin{proof}
First, assume that $p>0$ and let $\tilde S= \mathcal{\tilde{D}}^{-1}\tilde{\mathcal{W}}$. We have 
\begin{align*}
\tilde{\mathcal{L}}_{\sym}^p&=( \tau I_3 - \tilde{\mathcal{D}}^{-1/2}\tilde{\mathcal{W}}\tilde{\mathcal{D}}^{-1/2})^p \\&= \sum_{r =0}^p \binom{p}{r} \tau^{p-r}(-1)^r (\tilde{\mathcal{D}}^{-1/2}\tilde{\mathcal{W}}\tilde{\mathcal{D}}^{-1/2})^r \\
&= \tilde{\mathcal{D}}^{1/2}\Big(\sum_{r =0}^p \binom{p}{r} \tau^{p-r}(-1)^r (\tilde{\mathcal{D}}^{-1}\tilde{\mathcal{W}})^r\Big)\tilde{\mathcal{D}}^{-1/2}\\ &= \tilde{\mathcal{D}}^{1/2}(\tau I_3-\tilde S)^p\tilde{\mathcal{D}}^{-1/2}.
\end{align*} 
As $\tilde{\mathcal{D}}^{1/2}$ and $\tilde{\mathcal{D}}^{-1/2}$ are diagonal with positive diagonal entries, the sign of the entries of $\tilde{\mathcal{L}}_{\sym}^p$ coincide with those of $(\tau  I_3-\tilde S)^p$. Furthermore, we have $(\tilde{\mathcal{L}}_{\sym}^p)_{i,i}=((\tau  I_3-\tilde S)^p)_{i,i}$ for all $i$. Now the matrix $\tilde S$ is row stochastic, that is $\tilde S\mathbf 1=\mathbf 1$ and has the following form
$$
\tilde S = \begin{pmatrix}
1-2\hat{a} & \hat{a} & \hat{a} \\
1-2\hat{b} & \hat{b} & \hat{b} \\
1-2\hat{b} & \hat{b} & \hat{b}
\end{pmatrix}\quad \hat{a} = \frac{\hat{\alpha}}{1+2\hat{\alpha}}, \quad  \hat{b} = \frac{1}{2+\hat{\alpha}}
$$

where $\hat{\alpha} = \pout/\pin \in (0,1)$.
Let
$$\gamma = (\hat{a}-\hat{b})=\frac{\pout^2-\pin^2}{(2\pin+\pout)(2\pout+\pin)}<0,$$ $$ \mu=(1-2\hat{b})=\frac{\pin}{2\pout+\pin}>0$$
For all positive integer $p$, we have
\begin{equation*}
(\tau I_3-\tilde S)^p=\frac{1}{2\gamma+1}\begin{pmatrix} q_p & r_p & r_p \\ s_p & t_p & u_p \\ s_p & u_p & t_p\end{pmatrix}
\end{equation*}
where $q_p,r_p,s_p,t_p,u_p$ are given by
\begin{align}\label{coeffdef}
q_p&=\mu (\tau-1)^p+2 \hat{a} (2\gamma+\tau)^p, \notag\\
r_p &= \hat{a} \big[(\tau-1)^p-(2\gamma+\tau)^p\big],\notag\\
s_p &=  \frac{\mu}{\hat{a}}r_p,\\
t_p &=  \hat{a} \big[(\tau-1)^p+\tau^p\big]+\frac{\mu}{2}\big[\tau^p+(2 \gamma+\tau)^p\big],\notag\\
u_p& = \hat{a} \big[(\tau-1)^p-\tau^p\big]-\frac{\mu}{2}\big[\tau^p-(2 \gamma+\tau)^p\big].\notag
\end{align}
Note that  as $\pin>\pout>0$, we have
\begin{align*}\delta &= 2\gamma+1 =2(\hat{a}-\hat{b})+1 = 2\hat{a}+\mu\\ &=\frac{ 5 \pin \pout+4 \pout^2}{2 \pin^2+5\pin \pout+2\pout^2}\in(0,1),
\end{align*}
Furthermore, as $\tau\geq 1$ and $\gamma<0$, we have $\delta\leq (2\gamma +\tau)<\tau$. 
It follows that 
\begin{align*}
0&<\mu(\tau-1)^p+2 \hat{a}\delta^p\leq q_p<\mu(\tau-1)^p+2\hat{a}\tau^p \\ &\leq \delta \tau^p<\tau^p \\
0&<\hat{a} \big[(\tau-1)^p+\tau^p\big]+\frac{\mu}{2}(\tau^p+\delta^p)\leq t_p\\ &<2\hat{a}\tau^p+\mu\tau^p=\delta\tau^p<\tau^p.
\end{align*}
Finally, we have 
$$r_p=\hat{a} \big[(\tau-1)^p-\big(\delta+(\tau-1)\big)^p\big]<0, \quad s_p=\frac{\mu}{\hat{a}}r_p <0.$$
Now, suppose that $p<0$, then we have $\tau>1$ and $$(\tau I_3-\tilde{\mathcal{D}}^{-1/2}\tilde{\mathcal{W}}\tilde{\mathcal{D}}^{1/2})^{-1}=\sum_{k=0}^\infty (\tilde{\mathcal{D}}^{-1/2}\tilde{\mathcal{W}}\tilde{\mathcal{D}}^{1/2})^k.$$
As $\tilde{\mathcal{M}}=\tilde{\mathcal{D}}^{-1/2}\tilde{\mathcal{W}}\tilde{\mathcal{D}}^{1/2}$ is a matrix with strictly positive entries, this implies that $\tilde{\mathcal{L}}_{\sym}$ has positive entries as well. Furthermore, it also implies that $\tilde{\mathcal{L}}_{\sym}^p=(\tilde{\mathcal{L}}_{\sym}^{-1})^{|p|}$ is positive for every $p<0$. 
\end{proof}
\begin{observation}
Numerical evidences strongly suggest that the formulas in \eqref{coeffdef} for the coefficients of $(\tau I_3-\tilde S)^{p}$ hold for any real $p\in\R\setminus\{0\}$.
\end{observation}
We can now use the above lemma to determine the ordering of the eigenvalues of $\Lnp{p}$.
\begin{lemma}\label{orderLp}
Let $t_1,t_2\in\R$ be such that it holds $\Lnp{p}=(t_1-t_2)I_3+t_2\tilde E$. Furthermore, for any nonzero integer $p$, it holds $0<t_1-t_2<t_1+2t_2$ if $p<0$ and $t_1-t_2>t_1+2t_2$ otherwise.
\end{lemma}
\begin{proof}
If $p<0$, then we must have $\tau >1$ for $\Lnp{p}$ to be well defined. By Lemma \ref{shitmatrix}, $(\lsymn{1})^p$ has strictly positive entries. Hence, $\Lnp{p}=\sum_{t=1}^3\tilde P_t(\lsymn{1})^{p}\tilde P_t$ is also a matrix with positive entries. It follows that $t_1-t_2>0$ and $t_2>0$ so that $0<t_1-t_2<t_1+2t_2$.
Now assume that $p>0$, Lemma \ref{shitmatrix} implies that $(\lsymn{1})^p$ with positive diagonal elements and negative off-diagonal. It follows from \eqref{symact} that $\Lnp{p}$ also has positive diagonal elements and negative off-diagonal. Hence, we have $t_2<0<t_1$ and thus $t_1-t_2>t_1+2t_2$ which concludes the proof.
\end{proof}
We have the following corollary on the spectral properties of the Laplacian $p$-mean.
\begin{corollary}\label{specLmean3}
Let $p$ be a nonzero integer and let $\epsilon\geq 0$ if $p>0$ and $\epsilon >0$ if $p<0$. Define
$$\Ln{p}=\Big(\frac{(\lsymn{1})^p+(\lsymn{2})^p+(\lsymn{3})^p}{3}\Big)^{1/p},$$
then there exists $0\leq \tilde\lambda_1<\tilde\lambda_2$ such that the eigenpairs of $\Ln{p}$ are given by
\begin{align*}
&\big(\tilde\lambda_1,(-1,0,1)^\top\big),\qquad \big(\tilde\lambda_1,(-1,1,0)^\top\big),\\ &\big(\tilde\lambda_2,(1,1,1)^\top\big).
\end{align*}
\end{corollary}
\begin{proof}
First, note that $\Ln{p}=\big(\tfrac{1}{3}\Lnp{p}\big)^{1/p}$ hence as they are positive semi-definite matrices, $\Ln{p}$ and $\Lnp{p}$ share the same eigenvectors. Precisely, we have $\Lnp{p}\v=\lambda \v$ if and only if $\Ln{p}\v=f(\lambda)\v$ where $f(t)=(t/3)^{1/p}$. Now, by Lemmas \ref{LpinZ3} and \ref{specZ3} we know all eigenvectors of $\Lnp{p}$ and the corresponding eigenvalues are $\theta_1=t_1-t_2$ and $\theta_2 = t_1+2t_2$. Finally, using Lemma \ref{orderLp} and the fact that $f$ is increasing if $p>0$ and decreasing if $p<0$ we deduce the ordering of $\tilde \lambda_i=f(\theta_i)$.
\end{proof}

\underline{\textit{The case $n>1$:}}

We now generalize the previous results to the case $n>1$. 
To this end, we use mainly the properties of the Kronecker product $\otimes$ which we recall is defined for matrices 
$A\in\R^{m_1\times m_2},B\in\R^{m_3\times m_4}$
as the block matrix 
$A\otimes B\in\R^{m_1m_3\times m_2m_4}$
with 
$m_1m_2$ 
blocks of the form $A_{i,j}B\in\R^{m_3\times m_4}$ for all $i,j$. In particular, for $n>1$, if $E$ denotes the matrix of all ones in $\R^{n\times n}$, we have then $\Wt=\tilde{\mathcal{W}}^{(t)}\otimes E$ for every $t=1,2,3$. Furthermore, let us define $\mathcal{W}=\tilde{\mathcal{W}}\otimes E$ and $P_t=\tilde P_t\otimes I_n$ for $t=1,2,3$ so that $\Wt=P_t\mathcal{W}P_t$ for $t=1,2,3$. Finally, let $\mathcal{L}_{\sym}=\tau I_{3n}-\mathcal{D}^{-1/2}\mathcal{W}\mathcal{D}^{-1/2}$ where we recall that $\tau = 1+\epsilon$ and $\mathcal{D}=\diag(\mathcal{W}\ones)$.
The normalized Laplacians of $\mathcal{W}$ and $\tilde{\mathcal{W}}$ are related in the following lemma:
\begin{lemma}\label{kronL}
It holds
$$\mathcal{L}_{\sym}=\tau I_{3n} -\big[\tfrac{1}{n}\tilde{\mathcal{D}}^{-1/2}\tilde{\mathcal{W}}\tilde{\mathcal{D}}^{-1/2}\big]\otimes E.$$
\end{lemma}
\begin{proof}
First, note that $\mathcal{D}=n \tilde{\mathcal{D}}\otimes I_n$, as $(A_1\otimes B_1)(A_2\otimes B_2)=(A_1A_2\otimes B_1B_2)$ for any compatible matrices $A_1,A_2,B_1,B_2$. We have
\begin{align*}
\mathcal{D}^{-1/2}\mathcal{W}\mathcal{D}^{-1/2} &=\frac{(\tilde{\mathcal{D}}^{-1/2}\otimes I_n)(\tilde{\mathcal{W}}\otimes E)(\tilde{\mathcal{D}}^{-1/2}\otimes I_n)}{n}\\ &=\tfrac{1}{n}\tilde{\mathcal{D}}^{-1/2}\tilde{\mathcal{W}}\tilde{\mathcal{D}}^{-1/2}\otimes E,
\end{align*}
which concludes the proof.
\end{proof}
In order to study the eigenpairs of $\mathcal{L}_{\sym}$, we combine Lemma \ref{LsymW3} with the following theorem from \cite{Kronref} which implies that eigenpairs of Kronecker products are Kronecker products of the eigenpairs:
\begin{thm}[Theorem 4.2.12, \cite{Kronref}]\label{specKron}
Let $A\in\R^{m \times m}$ and $B\in\R^{n\times n}$. Let $(\lambda,x)$ and $(\mu,y)$ be eigenpairs of $A$ and $B$ respectively. Then $(\lambda\mu,x\otimes y)$ is an eigenpair of $A\otimes B$.
\end{thm}
Indeed, the above theorem implies that the eigenpairs of 
\pedro{${\mathcal{D}}^{-1/2} \mathcal{W} {\mathcal{D}}^{-1/2}$}
are Kronecker products of the eigenpairs of 
\pedro{${\mathcal{\tilde D}}^{-1/2} \tilde{\mathcal{W}} {\mathcal{ \tilde D}}^{-1/2} $}
and $E$. As we already know those of 
\pedro{${\mathcal{\tilde D}}^{-1/2} \tilde{\mathcal{W}} {\mathcal{ \tilde D}}^{-1/2} $}, 
we briefly describe those of~$E$:
\begin{lemma}\label{specE}
Let $E\in\R^{n\times n}$, $n\geq 2$ be the matrix of all ones, then the eigenpairs of $E$ are given by $(n,\ones)$ and $(0,\v_1),\ldots,(0,\v_{n-1})$ where $\v_k\in\R^n$ is given as
\begin{equation}\label{defvi2}
(\v_{k})_j=\begin{cases}1 & \text{if } j\leq k,\\ -k &\text{if } j=k+1,\\ 0 &\text{otherwise.}\end{cases}
\end{equation}
\end{lemma}
\begin{proof}
As $E=\ones\ones^\top$, it is clear that $(n,\ones)$ is an eigenpair of $E$. Now, for every $i$ we have $E\v_i=(\ones^\top v_i)\ones$ and $\ones^\top \v_i = i-i=0$.
\end{proof}
We can now describe the spectral properties of $\lsym{t}$ for $t=1,2,3$.
\begin{lemma}\label{LsymWn}
There exists $\lambda \in (0,1)$ and $s_-<0<s_+<1$ such that, for $t=1,2,3$, the eigenpairs of $\lsymn{t}$ are given by
\begin{align*}
&\big(\tau-1,P_t(s_+,1,1)^\top\otimes \ones\big), \, \big(\tau-\lambda,P_t(s_-,1,1)^\top\otimes \ones\big),\\
&\big(\tau,P_t(0,-1,1)^\top\otimes \ones\big),\qquad\!\big(\tau,P_t(s_+,1,1)^\top\otimes \v_k\big),\\
& \big(\tau,P_t(s_-,1,1)^\top\otimes \v_k\big), \qquad\!\! \big(\tau,P_t(0,-1,1)^\top\otimes \v_k\big),
\end{align*}
for $k=1,\ldots,n-1,$ where $\v_k$ is defined as in \eqref{defvi2}.
\end{lemma}
\begin{proof}
Follows from Lemmas \ref{specM}, \ref{specE} and Theorem \ref{specKron}.
\end{proof}
Similarly to the case $n=1$, let us consider $\Lp{p}\in\R^{3n\times 3n}$ defined as
$$\Lp{p}= (\lsym{1})^p+(\lsym{2})^p+(\lsym{3})^p=3\L{p}^p.$$
Again, we note that the eigenvectors of $\Lp{p}$ and $3\L{p}^p$ are the same. Now, let us consider the sets
 $\mathcal{U}_{3n}\subset\R^{3n\times 3n}$ and $\mathcal{Z}_{3n}\subset\R^{3n\times 3n}$ defined as
\begin{align*}
&\mathcal U_{3n}=\big\{s_0I_{3n}-\tilde A\otimes E\ \big| \ \tilde A\in\mathcal{U}_3, s_0\in\R\},\\
&\mathcal Z_{3n}=\big\{t_0I_{3n}-\tilde C\otimes E\ \big| \ \tilde C\in\mathcal{Z}_3, s_0\in\R\}.
\end{align*}
Note that, as $s_0I_3+\mathcal{U}_3=\mathcal{U}_3$ and $s_0I_3+\mathcal{Z}_3=\mathcal{Z}_3$ for all $s_0\in\R$, the definitions of $\mathcal{U}_{3n}$ and $\mathcal{Z}_{3n}$ reduce to that of $\mathcal{U}_{3}$ and $\mathcal{Z}_{3}$ when $n=1$. We prove that $\Lp{p}\in\mathcal{Z}_{3n}$ for all nonzero integer $p$. To this end, we first prove the following lemma which generalizes Lemma \ref{lemU}.
\begin{lemma} \label{lemUn}
The following holds:
\begin{enumerate}
\item $\mathcal U_{3n}$ is closed under multiplication, i.e. for all $A,B\in \mathcal U_{3n}$ we have $AB\in \mathcal U_{3n}$.\label{multiU3n}
\item If $A\in\mathcal U_{3n}$ satisfies $\det(A)\neq 0$, then $A^{-1}\in \mathcal U_{3n}$.\label{invU3n}
\item  $\mathcal{Z}_3=P_1\mathcal{U}_{3n} P_1+P_2\mathcal{U}_{3n} P_2+P_3\mathcal{U}_{3n} P_3$.\label{U3toZ3n}
\end{enumerate}
\end{lemma}
\begin{proof}
Let $ A, B\in \mathcal U_{3n}, C\in\mathcal{Z}_{3n}$ and $s_0, r_0,t_0\in\R$, $\tilde A,\tilde B\in\mathcal{U}_3,\tilde C\in\mathcal{Z}_3$ such that $ A = s_0I_{3n}-\tilde A\otimes E$, $ B =  r_0I_{3n}-\tilde B\otimes E$ and $ C=t_0I_{3n}-\tilde C\otimes E.$
\begin{enumerate}
\item We have 
\begin{align*}
 A\,  B 
&=s_0 r_0 I_{3n} +(n\tilde A\tilde B-s_0\tilde B- r_0\tilde A)\otimes E\end{align*}
As $\tilde A\tilde B\in\mathcal{U}_3$ 
by point \ref{multiU3} in Lemma \ref{lemU}, 
we have $(n\tilde A\tilde B-s_0\tilde B- r_0\tilde A)\in\mathcal{U}_3$ and so $ A B\in \mathcal U_{3n}$.
\item First note that as $ A$ is invertible, it holds $s_0\neq 0$. Furthermore, using \pedro{von} Neumann series, we have
\begin{align*}
(s_0I_{3n}-\tilde A\otimes E)^{-1}&=\sum_{k=0}^\infty s_0^{k-1} (\tilde A\otimes E)^k\\
&=\sum_{k=0}^\infty s_0^{k-1} n^k(\tilde A^k\otimes E).
\end{align*}
As $\tilde A^{k}\in\mathcal{U}_{3}$ for all $k$ 
by point \ref{multiU3} in Lemma \ref{lemU}
we have that $S_\nu = \sum_{k=0}^\nu s_0^{k-1} nk(\tilde A^k\otimes E)\in\mathcal{U}_{3n}$ for all $\nu=0,1,\ldots$ As $\lim_{\nu\to\infty}S_{\nu}= A^{-1}$ and $\mathcal{U}_{3n}$ is closed, it follows that $ A^{-1}\in\mathcal{U}_{3n}$.
\item Note that for $i=1,2,3$ it holds
\begin{align*}
 P_i\, A\, P_i & 
=s_0I_{3n}-(\tilde P_i\tilde A\tilde P_i\otimes E).
\end{align*}
Hence, we have
$$\sum_{i=1}^3 P_i\, A\, P_i
=3s_0I_{3n}-\Big(\sum_{i=1}^3\tilde P_i\tilde A\tilde P_i\Big)\otimes E.$$
We know from 
point \ref{U3toZ3} in Lemma \ref{lemU}
that $\sum_{i=1}^3\tilde P_i\tilde A\tilde P_i\in \mathcal{U}_3$ and thus $\sum_{i=1}^3 \tilde P_i\, \tilde A\, \tilde P_i\in\mathcal{Z}_3$. Finally, note that by choosing the coefficients in $\tilde A$ in the same way as in the proof of 
point \ref{U3toZ3} in Lemma \ref{lemU}, ,
we have $ A = C$ with $s_0=t_0$. This concludes the proof.
\end{enumerate}
\end{proof}
We can now prove that $\Lp{p}\in\mathcal{Z}_{3n}$.
\begin{lemma}\label{LpinZ3n}
For every nonzero integer $p$, we have $\Lp{p}\in\mathcal{Z}_{3n}$.
\end{lemma}
\begin{proof}
As $\mathcal{L}_{\sym}=\lsym{1}\in\mathcal{U}_{3n}$, we have $\mathcal{L}_{\sym}^p\in\mathcal{U}_{3n}$ by 
points \ref{multiU3n} and \ref{invU3n} in Lemma \ref{lemUn}.
We prove that $\Lp{p}=\sum_{t=1}^3P_t\mathcal{L}_{\sym}^pP_t$. To this end, note that, with the convention that powers on vectors are considered component wise, for $t=1,2,3$, we have 
\begin{align*}
\diag(P_t\mathcal{W} P_t\ones)^{1/2}
&=\diag\big(P_t(\mathcal{W} \ones)^{-1/2}\big)\\ &=P_t\diag\big((\mathcal{W}\ones)^{-1/2}\big)P_t=P_t\mathcal{D} P_t.
\end{align*}
Furthermore,
\begin{align*}
& \diag(P_t\mathcal{W} P_t\ones)^{-1/2}P_t\mathcal{W}P_t \diag(P_t\mathcal{W} P_t\ones)^{-1/2}\\&\quad=P_t \mathcal{D}^{-1/2}P_t^2\mathcal{W}P_t^2\mathcal{D}^{-1/2}P_t =P_t \mathcal{D}^{-1/2}\mathcal{W}\mathcal{D}^{-1/2}P_t.
\end{align*}
This implies that $\lsym{t}=P_t\mathcal{L}_{\sym}P_t$ for $t=1,2,3$ and thus we obtain the desired expression for $\Lp{p}$. 
Point \ref{U3toZ3n} in Lemma \ref{lemUn}
finally imply that $\Lp{p}\in\mathcal{Z}_{3n}$.  
\end{proof}
We combine Theorem \ref{specKron} and Lemmas \ref{specZ3}, \ref{specE} to obtain the following:
\begin{lemma}\label{specZ3n}
Let $C\in\mathcal{Z}_{3n}$ and $t_0,t_1,t_2$ such that 
$ C =t_0I_{3n}-((t_1-t_2)I_3+t_2\tilde E)\otimes E$. 
Then, the eigenpairs of $C$ are given by 
\begin{align*}
&\big(t_0-n(t_1-t_2),(-1,0,1)^\top\otimes \ones\big),\\ &\big(t_0-n(t_1-t_2),(-1,1,0)^\top\otimes \ones\big),\\ &\big(t_0-n(t_1+2t_2),(1,1,1)^\top\otimes \ones\big).
\end{align*}
and, with $\v_i$ defined as in \eqref{defvi2},
\begin{align*}
&\big(t_0,(-1,0,1)^\top\otimes \v_i\big), \qquad \big(t_0,(-1,1,0)^\top\otimes \v_i\big),\\ &\big(t_0,(1,1,1)^\top\otimes \v_i\big), \qquad i=1,\ldots,n-1.
\end{align*}
\end{lemma}
Similar to Lemma \ref{orderLp}, we have following lemma for deciding the order of the eigenvectors of $\Lp{p}$.
\begin{lemma}\label{shitmatrixn}
For every positive $p>0$ we have $(\mathcal{L}_{\sym}^p)_{i.j}<0<(\mathcal{L}_{\sym}^p)_{i.i}<\tau^p$ for all $i,j=1,\ldots,3n$ with $i\neq j$. For every negative $p<0$ we have $(\mathcal{L}_{\sym}^p)_{i,j}>0$ for all $i,j=1,\ldots,3n$.
\end{lemma}
\begin{proof}
Let $\mathcal{M}=\mathcal{D}^{-1/2}\mathcal{W}\mathcal{D}^{-1/2}$, then by Lemma \ref{kronL}, we have $\mathcal{M}=\tfrac{1}{n}(\tilde{\mathcal{M}}\otimes E)$.
Now, for $p>0$, it holds:
\begin{align}\label{usefuleq}
\mathcal{L}_{\sym}^p &= 
\tau^{p}I_{3n}+ \sum_{k=1}^p\binom{p}{k}\tau^{p-k}(-1)^kn^{-k}(\tilde{\mathcal{M}}^k\otimes E^k)\notag\\
&= \tau^{p}I_{3n}+\Big(\sum_{k=1}^p\binom{p}{k}\tau^{p-k}(-1)^k\tilde{\mathcal{M}}^k\Big)\otimes E\notag\\
&=\tau^{p}I_{3n}+\big(\tilde{\mathcal{L}}_{\sym}^p-\tau^p I_3\big)\otimes E.
\end{align}
By Lemma \ref{shitmatrix}, we know that $(\tilde{\mathcal{L}}_{\sym}^p)_{i,j}<0$ if $i\neq j$ and $(\tilde{\mathcal{L}}_{\sym}^p)_{i,i}-\tau^p<0$ for all $i$. Hence, the matrix $\tilde Q=\tilde{\mathcal{L}}_{\sym}^p-\tau^pI_3$ has strictly negative entries. Thus, all the off-diagonal elements of $\mathcal{L}_{\sym}^p$ are strictly negative. Finally, note that $$(\mathcal{L}_{\sym}^p)_{i,i}=\tau^p+(\tilde{\mathcal{L}}_{\sym}^p\otimes E)_{i,i}-\tau^p=(\tilde{\mathcal{L}}_{\sym}^p\otimes E)_{i,i}>0.$$
This concludes the proof for the case $p>0$. The case $p<0$ can be proved in the same way as for the case $n=1$ (see Lemma \ref{shitmatrix}). 
\end{proof}
\begin{observation}
We note that Equation \eqref{usefuleq} implies the following relation between $\Lp{p}$ and $\Lnp{p}$:
\begin{equation}\label{eq2use}
\Lp{p}=3\tau^{p}I_{3n}+\big(\Lnp{p}-\tau^p I_3\big)\otimes E.
\end{equation}
\end{observation}

\begin{lemma}\label{orderLnp}
Let $t_0,t_1,t_2\in\R$ be such that $\Lp{p}=t_0I_{3n}-((t_1-t_2)I_3+t_2E_3)\otimes E_n$. Furthermore, for any integer $p\neq 0$, it holds $t_0<t_0-n(t_1-t_2)<t_0-n(t_1+2t_2)$ if $p<0$ and $t_0>t_0-n(t_1-t_2)>t_0-n(t_1+2t_2)$ otherwise.
\end{lemma}
\begin{proof}
The proof is essentially the same as that of Lemma \ref{orderLp}. Indeed, if $p<0$, then $\Lp{p}$ is strictly positive and thus $t_2<0$ as $(\Lp{p})_{1,3n}>0$,  $t_1-t_2< 0$ as $(\Lp{p})_{1,n}>0$ and $t_0-nt_1>0$ as $(\Lp{p})_{1,1}>0$. This means that $t_1-t_2>t_1+2t_2$ and so $t_0-n(t_1-t_2)<t_0-n(t_1+2t_2)$. Furthermore, this shows that $t_0-n(t_1-t_2)>t_0$. Now, if $p>0$, by Lemma \ref{shitmatrixn} we have $t_2>0$ as $(\Lp{p})_{1,3n}<0$,  $t_1-t_2> 0$ as $(\Lp{p})_{1,n}<0$ and $t_0-nt_1>0$ as $(\Lp{p})_{1,1}>0$. It follows that $t_1-t_2<t_1+2t_2$ and thus $t_0-n(t_1-t_2)>t_0-n(t_1+2t_2)$. Finally, as $t_1-t_2>0$, we have $t_0>t_0-n(t_1-t_2)$ which concludes the proof.
\end{proof}
We conclude by giving a description of the spectral properties of $\L{p}$.
\begin{lemma}\label{specmainthm}
Let $p$ be any nonzero integer and assume that $\epsilon>0$ if $p< 0$. Define
$$\L{p}=\Big(\frac{(\lsym{1})^p+(\lsym{2})^p+(\lsym{3})^p}{3}\Big)^{1/p},$$
then there exists $0\leq \lambda_1,\lambda_2<\lambda_3$ 
such that all the eigenpairs of $\L{p}$ are given by
\begin{align*}
& \big(\lambda_1,(-1,0,1)^\top\otimes \ones\big), \qquad              \big(\lambda_3,(-1,0,1)^\top\otimes \v_i\big)\\ 
& \big(\lambda_1,(-1,1,0)^\top\otimes \ones\big), \qquad              \big(\lambda_3,(-1,1,0)^\top\otimes \v_i\big)\\
& \big(\lambda_2,(1,1,1)^\top\otimes \ones \big), \qquad\quad\!              \big(\lambda_3,(1,1,1)^\top\otimes \v_i\big),
\end{align*}

and $i=1,\ldots,n-1$,
where $\v_i$ is defined in \eqref{defvi2}.
\end{lemma}
\begin{proof}
The proof is the same as that of Corollary \ref{specLmean3} where one uses Lemmas \ref{LpinZ3n}, \ref{specZ3n}, \ref{orderLnp} instead of Lemmas \ref{LpinZ3}, \ref{specZ3}, \ref{orderLp}.
\end{proof}

\end{document}